\documentclass[sn-mathphys,Numbered]{sn-jnl}

\usepackage{graphicx}
\usepackage{subfigure}
\usepackage{multirow}%
\usepackage{amsmath,amssymb,amsfonts}%
\usepackage{amsthm}%
\usepackage{mathrsfs}%
\usepackage[title]{appendix}%
\usepackage{xcolor}%
\usepackage{textcomp}%
\usepackage{manyfoot}%
\usepackage{booktabs}%
\usepackage{algpseudocode}%
\usepackage{listings}%
\usepackage{matlab-prettifier}
\usepackage{changes}
\RequirePackage{algorithm,algpseudocode}
\usepackage{rotating}
\usepackage{tikz}
\usetikzlibrary{arrows, decorations.pathmorphing, backgrounds, positioning, fit, petri, automata}
\usetikzlibrary{shadows,arrows,positioning}

\newtheorem{thm}{Theorem}

\newtheorem{lem}{Lemma}

\newtheorem{rem}{Remark}

\begin{document}

\title[Fast estimation of Gaussian mixture components via centering and singular value thresholding]{Fast estimation of Gaussian mixture components via centering and singular value thresholding}

\author*[1]{\fnm{Huan} \sur{Qing}}\email{qinghuan@u.nus.edu}
\affil[1]{\orgdiv{School of Economics and Finance}, \orgname{Chongqing University of Technology}, \city{Chongqing}, \postcode{400054}, \state{Chongqing}, \country{China}}


\abstract{Estimating the number of components is a fundamental challenge in unsupervised learning, particularly when dealing with high-dimensional data with many components or severely imbalanced component sizes. This paper addresses this challenge for classical Gaussian mixture models. The proposed estimator is simple: center the data, compute the singular values of the centered matrix, and count those above a threshold. No iterative fitting, no likelihood calculation, and no prior knowledge of the number of components are required. We prove that, under a mild separation condition on the component centers, the estimator consistently recovers the true number of components. The result holds in high-dimensional settings where the dimension can be much larger than the sample size. It also holds when the number of components grows to the smaller of the dimension and the sample size, even under severe imbalance among component sizes. Computationally, the method is extremely fast: for example, it processes ten million samples in one hundred dimensions within one minute. Extensive experimental studies confirm its accuracy in challenging settings such as high dimensionality, many components, and severe class imbalance.
}

\maketitle
\section{Introduction}\label{sec:intro}
A fundamental challenge in unsupervised learning is to discover the latent group structure of unlabelled data \citep{jain1999data,jain2010data,gormley2023model}. Consider, for instance, a medical study where hundreds of thousands of patient records contain many clinical measurements. Without any diagnostic labels, one may ask whether the patients naturally divide into a few subgroups corresponding to distinct disease stages or risk profiles. Similarly, a large collection of handwritten digit images, each represented by a high‑dimensional vector of pixel intensities, raises the question of how many distinct digit classes are present. In both examples the number of underlying groups is unknown, and the data can be massive, with the number of variables sometimes far exceeding the number of samples or vice versa.

The Gaussian mixture model (GMM) provides a classical and effective model-based unsupervised clustering framework for such problems \citep{pearson1894iii}. It assumes that each observation is drawn from one of several Gaussian components, each with its own mean and covariance. GMMs are widely used because they can approximate any continuous distribution arbitrarily well and offer a probabilistic interpretation of cluster membership \citep{lindsay1995mixture,mclachlan2000finite,fop2018variable}. The standard tool for fitting a GMM is the expectation‑maximisation (EM) algorithm \citep{dempster1977maximum}, which iteratively estimates the component parameters. However, EM requires the number of components \(K\) to be specified in advance. Recent years, a large body of recent work has established sharp theoretical guarantees for clustering under the GMM. For example, \citep{lu2016statistical} proved that Lloyd's algorithm, when properly initialized, attains an exponential rate of misclassification. \citep{loffler2021optimality} demonstrated that spectral clustering is minimax optimal for fixed \(K\). \citep{chen2021cutoff} and \citep{ndaoud2022sharp} derived sharp cutoffs for exact recovery using semidefinite programming and a hollowed Gram matrix, respectively, while \citep{li2025exact} further characterised the exact recovery threshold for mixtures with entrywise heterogeneous noise. \citep{chen2024achieving} extended the theory to anisotropic Gaussian mixtures where clusters have different covariances. It should be emphasized that, in all these state‑of‑the‑art theoretical developments under GMMs, a key quantity that repeatedly appears is the minimum distance between any two distinct component centres, denoted by \(\Delta\) (see its formal definition in Equation (\ref{Defin:Delta})). This parameter directly governs the difficulty of the clustering problem: when \(\Delta\) is large, the clusters are well separated and easy to distinguish; when \(\Delta\) is small, the clusters overlap and become nearly impossible to separate even with a large sample size. The theoretical conditions in these advanced works cited above are invariably expressed in terms of \(\Delta\) relative to the noise level and the sample size, highlighting its fundamental role. Despite these substantial advances, all of the above methods require the true number of components \(K\) to be known in advance.

In practice, however, \(K\) is rarely known in advance. A common heuristic is to run EM over a grid of candidate values and then select one using an information criterion such as AIC \citep{akaike2003new} or BIC \citep{schwarz1978estimating}. This approach is computationally burdensome in high dimensions, lacks theoretical guarantees for consistency, and often fails when the data are large or imbalanced.  Over the years, many methods have been developed to estimate \(K\) automatically. Well‑known examples include the Gap statistic \citep{tibshirani2001estimating}, the X‑means algorithm \citep{pelleg2000x}, the G‑means algorithm \citep{hamerly2003learning}, the PG‑means algorithm \citep{feng2006pg},  EM‑based model selection approaches \citep{melnykov2012initializing}, Bayesian approaches \citep{corduneanu2001variational,constantinopoulos2006bayesian,bishop2006pattern,malsiner2016model,teklehaymanot2018bayesian}, and penalized likelihood techniques \citep{huang2017model,budanova2025penalized}. In contrast to the state‑of‑the‑art clustering results discussed above, which explicitly characterize the required separation \(\Delta\) between component centres, none of these methods for estimating \(K\) provides a theoretical condition on \(\Delta\) that guarantees consistent recovery. Instead, their theoretical guarantees are typically derived under other restrictive assumptions. Nevertheless, most of these approaches are iterative, relying on the EM algorithm or similar iterative procedures, and therefore become computationally prohibitive when the sample size \(n\) or the data dimension \(p\) reaches the order of millions. Moreover, no existing method provides a simple non‑iterative estimator that provably attains the fundamental identifiability bound \(K \le \min(p,n)\) without requiring a priori upper bound on \(K\) or assuming \(K\) fixed, and existing consistency results seldom incorporate explicit conditions on the balance of cluster sizes even though this quantity also critically governs the difficulty of the estimation problem. As a consequence, such methods frequently overfit or underfit in high‑dimensional, severely imbalanced, or large‑scale settings.

In this paper, we consider the problem of estimating the true number of components \(K\) in the classical Gaussian mixture model. Our goal is to overcome the limitations of existing methods: they are iterative, lack explicit conditions on the centre separation \(\Delta\), require an upper bound on \(K\) or assume it fixed, and cannot reach the fundamental limit \(K \le \min(p,n)\). Our main contributions are as follows.

\begin{itemize}
    \item \textbf{A non‑iterative algorithm based on centering and singular value thresholding.} The estimator centres the data matrix, computes the singular values of the centred matrix, and counts how many of them exceed a threshold determined by the noise level. The procedure requires no iterative fitting, no likelihood evaluation, and no prior specification of the true number of components \(K\). The only tuning parameter is a sequence that tends to infinity slowly, and its specific choice has little effect on the result.
    
    \item \textbf{Theoretical consistency with explicit conditions and a justification of centering.} For the classical Gaussian mixture model, we prove that the estimator recovers the true number of components \(K\) with probability approaching one, under a mild separation condition on the component centres. The separation condition is stated explicitly in terms of the centre separation and the cluster balance parameters. The theory allows \(K\) to grow up to the maximal possible value \(\min(p,n)\), covers both the high‑dimensional regime (\(p \gg n\)) and the sample‑rich regime (\(n \gg p\)), and accommodates severely imbalanced clusters. In addition, we construct counterexamples to show that the centering step is not optional: without it, consistent estimation fails in certain simple settings. This work thus provides a practical and theoretically sound solution to a long‑standing challenge for Gaussian mixture models in unsupervised learning.
    
    \item \textbf{Empirical evaluation on synthetic and real data.} We investigate the performance of the proposed estimator on a range of challenging synthetic scenarios, including high dimensionality, large numbers of clusters (up to \(\min(p,n)\)), and severe class imbalance (e.g., the smallest cluster contains only 20 observations out of one million samples). The method processes ten million samples in one hundred dimensions and returns the estimated \(K\) within one minute. We also evaluate the estimator on several real‑world benchmark datasets, where it either recovers the true number of clusters or provides a meaningful description of the data geometry.
\end{itemize}

The remainder of the paper is organised as follows. Section~\ref{sec:model} formally introduces the model. Section~\ref{sec:alg} presents the algorithm. Section~\ref{sec:theory} provides the main theoretical results and also justifies the necessity of centering. Section~\ref{sec:experiments} conducts simulation studies. Section~\ref{sec:realdata} tests the estimator on real data. Section~\ref{sec:conclusion} concludes this paper with discussions of future directions.

\paragraph*{Notation}
For a vector \(\mathbf{v}\), \(\|\mathbf{v}\|\) denotes its Euclidean norm. For a matrix \(\mathbf{A}\), \(\|\mathbf{A}\|\) denotes its spectral norm (largest singular value), and \(\sigma_i(\mathbf{A})\) is its \(i\)-th largest singular value. \(\mathbf{A}^\top\) is the transpose. \(\mathcal{N}(\boldsymbol{\mu},\boldsymbol{\Sigma})\) represents a multivariate Gaussian distribution with mean \(\boldsymbol{\mu}\) and covariance \(\boldsymbol{\Sigma}\).  For any positive integer \(m\): \([m] = \{1,2,\dots,m\}\), \(\operatorname{diag}(a_1,\dots,a_m)\) is the diagonal matrix with entries \(a_1,\dots,a_m\), \(\mathbf{I}_m\) stands for the \(m\times m\) identity matrix, and \(\mathbf{1}_m\) is the \(m\)-dimensional all-ones vector. \(\mathbf{e}_i\) is the \(i\)-th standard basis vector of appropriate dimension. \(\mathbb{E}(\cdot)\) and \(\mathbb{P}(\cdot)\) denote expectation and probability, respectively. 
\section{Gaussian mixture model}\label{sec:model}

Consider $n$ independent observations drawn from a mixture of $K$ Gaussian components.  For each observation $i$, let $\mathbf{z}_i\in\{1,2,\dots,K\}$ be its unknown component label. In this paper, we consider the following classical Gaussian mixture model:
\begin{align}\label{def:GMM}
\mathbf{x}_i = \boldsymbol{\mu}_{\mathbf{z}_i} + \boldsymbol{\varepsilon}_i,\qquad 
\boldsymbol{\varepsilon}_i\stackrel{\text{i.i.d.}}{\sim}\mathcal{N}(\mathbf{0},\mathbf{I}_p),
\end{align}
where $\boldsymbol{\mu}_1,\boldsymbol{\mu}_2,\dots,\boldsymbol{\mu}_K\in\mathbb{R}^p$ are the component centres.  The number of components $K\ge 1$ is the primary target of inference.

Stack the observations into the $p\times n$ data matrix $\mathbf{X}=[\mathbf{x}_1,\mathbf{x}_2,\dots,\mathbf{x}_n]$, the centres into the $p\times K$ centers matrix $\mathbf{M}=[\boldsymbol{\mu}_1,\boldsymbol{\mu}_2,\dots,\boldsymbol{\mu}_K]\in\mathbb{R}^{p\times K}$ (full column rank), and the membership indicators into the $n\times K$ membership matrix $\mathbf{Z}\in\{0,1\}^{n\times K}$ with $\mathbf{Z}_{ik}=1$ if and only if the $i$‑th observation belongs to the $k$‑th cluster..  The model admits the compact representation
\[
\mathbf{X} = \mathbf{M}\mathbf{Z}^{\top} + \mathbf{E},
\]
with $\mathbf{E}$ containing i.i.d. $\mathcal{N}(0,1)$ entries. For convenience, set $\mathbf{P} = \mathbf{M}\mathbf{Z}^{\top}$. Because $\mathrm{rank}(\mathbf{M})=K$ and $\mathrm{rank}(\mathbf{Z})=K$ since each cluster has at least one observation, we see that $\mathrm{rank}(\mathbf{P})=K$. 

In the broader context of unsupervised learning, a fundamental goal is to uncover hidden structure from unlabelled data.  When a classical Gaussian mixture model is adopted and the number of components \(K\) is known in advance, a task of central interest is to recover the latent component labels—that is, to cluster the observations.  The difficulty of this clustering problem is governed primarily by the distances between the component centres \citep{loffler2021optimality,ndaoud2022sharp}.  When the centres are well separated, the components are easier to distinguish. Conversely, if two centres lie close to one another, the corresponding clusters become nearly indistinguishable even with a large sample.  The key quantity that captures this separation is the minimal pairwise distance among all centres:
\begin{align}\label{Defin:Delta}
\Delta = \min_{k\neq\ell}\|\boldsymbol{\mu}_k-\boldsymbol{\mu}_\ell\|.
\end{align}

A larger \(\Delta\) implies a more favourable clustering problem, whereas a smaller \(\Delta\) demands more data or stronger assumptions to achieve reliable recovery.

Beyond the minimal centre separation, another critical quantity that governs the difficulty of the estimation problem is the balance among the cluster sizes.  When clusters are severely imbalanced, the smallest cluster may contain very few observations, making it harder to identify its presence.  To quantify this aspect, let $n_k = |\{i:\mathbf{Z}_{ik}=1\}|$ be the size of the $k$th cluster and define the balance parameter
\[
\beta = \frac{\min_k n_k}{n/K}\in(0,1].
\]

Thus $\min_k n_k \ge \beta n/K$, with $\beta=1$ corresponding to perfectly balanced clusters.  Importantly, the theoretical analysis in this paper allows $\beta$ to tend to zero as $n$ grows, thereby accommodating arbitrarily unbalanced clusters without imposing a lower bound on the smallest cluster size. 

We operate in the asymptotic regime $n\to\infty$; the dimension $p=p_n$ may increase with $n$ as long as $\sqrt{p_n}+\sqrt{n}\to\infty$ (this includes both fixed $p$ and $p\to\infty$).  No further restrictions are placed on $K$ beyond the identifiability necessity bound $K\le\min(p,n)$.

Given only the data matrix \(\mathbf{X}\), the latent class labels \(\mathbf{z}_i\) cannot be recovered without knowing the true number of components \(K\).  Estimating \(K\) is therefore a more fundamental and challenging task than clustering itself in unsupervised learning: it must be resolved before any meaningful assignment of observations to groups can take place.  This problem lies at the very heart of model selection in unsupervised learning.  In this work, we focus precisely on this problem within the classical Gaussian mixture model.  The next section introduces a simple estimator that requires no prior knowledge of \(K\) and is provably consistent under mild conditions.
\section{Centered singular value thresholding algorithm}
\label{sec:alg}

We begin with a simple observation. If we knew the true global mean (which is $\boldsymbol{\mu}_0 = \frac{1}{n}\sum_{k=1}^K n_k \boldsymbol{\mu}_k$) of the component centres, and subtracted it from every observation, then the centred signal matrix would have rank \(K-1\) when there are two or more components. When there is only one component, the centred signal would be exactly zero. Therefore, the number of nonzero singular values of that ideal centred matrix would directly tell us \(K-1\). In practice, we do not know the global mean, but we can estimate it by the sample mean. Centering the data with the sample mean reduces the rank of the signal from \(K\) to \(K-1\) (for \(K\ge2\)) and leaves the noise almost unchanged. Without centering, the signal would keep rank \(K\). Centering reduces the rank to \(K-1\) and removes the dominant direction, making the remaining signal singular values easier to separate from the noise. As will be rigorously demonstrated in Section~\ref{Whycentering} through explicit counterexamples, centering offers provable advantages over the uncentered approach. So centering is the key first step.

Recall that \(\mathbf{X}\) is the \(p \times n\) data matrix and \(\mathbf{X} = \mathbf{P} + \mathbf{E}\), where \(\mathbf{P}\) contains the component centres and \(\mathbf{E}\) contains independent standard normal noise. Let \(\bar{\mathbf{x}}\) be the sample mean of the columns; note that \(\mathbb{E}(\bar{\mathbf{x}}) = \boldsymbol{\mu}_0\), the true global mean of the component centres. Define the centred data matrix
\[
\widetilde{\mathbf{X}} = \mathbf{X} - \bar{\mathbf{x}} \mathbf{1}_n^\top = \mathbf{X}\mathbf{H},
\quad \mathbf{H} = \mathbf{I}_n - \tfrac{1}{n}\mathbf{1}_n\mathbf{1}_n^\top,
\]
where $\mathbf{H}$ is the centering projection matrix, , which subtracts the sample mean from each column. Correspondingly, the centred signal and centred noise are \(\widetilde{\mathbf{P}} = \mathbf{P}\mathbf{H}\) and \(\widetilde{\mathbf{E}} = \mathbf{E}\mathbf{H}\). Then \(\widetilde{\mathbf{X}} = \widetilde{\mathbf{P}} + \widetilde{\mathbf{E}}\). Now, we compute the singular values of \(\widetilde{\mathbf{X}}\). Because \(\mathbf{H}\) is an orthogonal projection, \(\|\mathbf{H}\|=1\) and thus \(\|\widetilde{\mathbf{E}}\| \le \|\mathbf{E}\|\). A standard result for Gaussian random matrices (Lemma~\ref{lem:noise}) tells us that with high probability, we have
\[
\|\mathbf{E}\| \le \sqrt{p} + \sqrt{n} + t_n,
\]
where \(t_n\) is any sequence satisfying \(t_n\to\infty\) and \(t_n = o(\sqrt{p}+\sqrt{n})\) (for example \(t_n = \log n\)). Hence the same bound holds for \(\|\widetilde{\mathbf{E}}\|\). Set the threshold
\[
T = \sqrt{p} + \sqrt{n} + t_n.
\]

If there is only one component, then \(\widetilde{\mathbf{P}} = \mathbf{0}\) and \(\widetilde{\mathbf{X}} = \widetilde{\mathbf{E}}\). Hence, all singular values of \(\widetilde{\mathbf{X}}\) are below \(T\) with probability tending to one. If there are two or more components and the centres are well separated as made precise in Theorem~\ref{thm:main} provided in the next section, then the smallest positive singular value of \(\widetilde{\mathbf{P}}\) is larger than \(2T\). By Weyl's inequality, the first \(K-1\) singular values of \(\widetilde{\mathbf{X}}\) then exceed \(T\) while the \(K\)-th singular value stays below \(T\). Therefore, the number of singular values of \(\widetilde{\mathbf{X}}\) that are greater than \(T\) is exactly \(K-1\) with high probability. Adding one gives a consistent estimate of \(K\). Algorithm \ref{alg:CSVT} provided below summarizes this idea into practice.

\begin{algorithm}
\caption{Centered singular value thresholding (CSVT)}
\label{alg:CSVT}
\begin{algorithmic}[1]
\Require Data matrix \(\mathbf{X}\in\mathbb{R}^{p\times n}\)
\Ensure Estimated number of components \(\widehat{K}\)
\State Compute the sample mean \(\bar{\mathbf{x}} = \frac{1}{n}\sum_{i=1}^n \mathbf{X}_{:,i}\)
\State Form the centred matrix \(\widetilde{\mathbf{X}} = \mathbf{X} - \bar{\mathbf{x}}\mathbf{1}_n^\top\)  (equivalently \(\widetilde{\mathbf{X}} = \mathbf{X}\mathbf{H}\))
\State Compute the singular values \(\widehat{\sigma}_1 \ge \cdots \ge \widehat{\sigma}_m\) of \(\widetilde{\mathbf{X}}\) with \(m = \min(p,n)\)
\State Choose a sequence \(t_n\) such that \(t_n\to\infty\) and \(t_n=o(\sqrt{p}+\sqrt{n})\) (e.g., \(t_n = \log n\)) and set \(T = \sqrt{p}+\sqrt{n}+t_n\)
\State Count \(r = |\{ i : \widehat{\sigma}_i > T \}|\)
\State \Return \(\widehat{K} = r+1\)
\end{algorithmic}
\end{algorithm}

The computational cost of CSVT consists of three parts. Centering the data matrix subtracts the column mean from each entry, which takes $\mathcal{O}(pn)$ time and requires $\mathcal{O}(pn)$ space to store the matrix. Computing the singular value decomposition (SVD) of the resulting $p\times n$ centred matrix takes $\mathcal{O}(p n \min(p,n))$ time; the space needed for the SVD is also $\mathcal{O}(pn)$. Counting the number of singular values above the threshold adds $\mathcal{O}(\min(p,n))$ time. Hence, CSVT's overall time complexity is $\mathcal{O}(p n \min(p,n))$, and its space complexity is $\mathcal{O}(pn)$.
\section{Theoretical analysis}
\label{sec:theory}

We now turn to the theoretical properties of the proposed estimator. Recall the centered data matrix $\widetilde{\mathbf{X}} = \mathbf{X}\mathbf{H}$ with $\mathbf{H}=\mathbf{I}_n-\frac1n\mathbf{1}_n\mathbf{1}_n^\top$, and the decomposition $\widetilde{\mathbf{X}} = \widetilde{\mathbf{P}} + \widetilde{\mathbf{E}}$ where $\widetilde{\mathbf{P}}=\mathbf{P}\mathbf{H}$ and $\widetilde{\mathbf{E}}=\mathbf{E}\mathbf{H}$. Since $\mathbf{H}$ is an orthogonal projection, we have $\|\widetilde{\mathbf{E}}\|\le\|\mathbf{E}\|$. The behavior of $\widetilde{\mathbf{P}}$ is fundamentally different depending on whether $K=1$ or $K\ge2$. If $K=1$, then all columns of $\mathbf{P}$ are identical, so centering removes the entire signal: $\widetilde{\mathbf{P}}=\mathbf{0}$. If $K\ge2$, the matrix $\widetilde{\mathbf{P}}$ has rank exactly $K-1$, and in particular its $K$-th singular value is zero: $\sigma_K(\widetilde{\mathbf{P}})=0$. This observation is the key to estimating $K$: we need to separate the $K-1$ non‑zero singular values of the signal from the noise.

A natural question is: how large are those non‑zero singular values? The following lemma gives a concrete lower bound in terms of the separation parameter $\Delta$ between component centres, the condition number of the centre matrix, and the cluster balance parameter $\beta$. 
\begin{lem}\label{lem:signal}
When $K\ge2$, then
\[
\sigma_{K-1}(\widetilde{\mathbf{P}}) \;\ge\; \frac{\Delta}{\kappa}\,\sqrt{\frac{\beta n}{2K}},
\]
where $\kappa$ is the condition number of $\mathbf{M}$, i.e., $\kappa = \frac{\sigma_1(\mathbf{M})}{\sigma_K(\mathbf{M})}$.
\end{lem}

Lemma \ref{lem:signal} says that the signal strength grows at least like $\sqrt{n}$ as long as $\Delta$ is bounded away from zero, $\kappa$ does not blow up, and the smallest cluster is not too tiny (i.e., $\beta$ does not vanish too fast). The parameter $\beta$ is allowed to tend to zero, but then the lower bound becomes smaller, which will later force a stronger separation condition.

The following theorem is the main theoretical result of this paper, which states that the proposed CSVT approach works with probability tending to one, under a separation condition that naturally emerges from the theoretical analysis.
\begin{thm}\label{thm:main}
Under the Gaussian mixture model defined in Equation (\ref{def:GMM}) and choose a sequence $t_n$ such that $t_n\to\infty$ and $t_n=o(\sqrt{p}+\sqrt{n})$ (e.g., $t_n=\log n$).  Define $d = \sqrt{p}+\sqrt{n}$ and the threshold $T = d + t_n$. Let $\widehat{K}$ be the output of Algorithm~\ref{alg:CSVT} (which uses the threshold $T$).  Then, the following hold.
\begin{itemize}
    \item If $K=1$, then $\displaystyle\lim_{n\to\infty}\mathbb{P}(\widehat{K}=1)=1$ without any further condition.
    \item If $K\ge2$ and the separation condition
    \[
    \Delta \;\ge\; 2\sqrt{2}\,\frac{\kappa}{\sqrt{\beta}}\,\sqrt{\frac{K}{n}}\;T
    \]
    holds, then $\displaystyle\lim_{n\to\infty}\mathbb{P}(\widehat{K}=K)=1$.
\end{itemize}
\end{thm}

The separation condition in Theorem~\ref{thm:main} reveals how the difficulty of estimating \(K\) depends on the model parameters.  
A smaller imbalance parameter \(\beta\) (meaning the smallest cluster is very small) forces a larger minimal distance \(\Delta\) between the component centres. This is intuitive: a tiny cluster is easy to overlook, so its centre must be placed further away from the others to be detected. Similarly, a larger condition number \(\kappa\) of the centre matrix \(\mathbf{M}\) indicates that the centres are nearly collinear, concentrating the signal in a low dimensional subspace.  
In that case the smallest nonzero singular value of the signal becomes more fragile, and a larger \(\Delta\) is required to keep it above the noise threshold. Finally, a larger number of components \(K\) also increases the required separation, because the signal is spread across more dimensions, making the weakest signal direction harder to distinguish from noise.

It should be emphasized that Theorem~\ref{thm:main} only requires the natural bound \(K \le \min(p,n)\). This condition is theoretically optimal: if \(K\) exceeds \(\min(p,n)\), then even in the noiseless case the signal matrix \(\mathbf{P}\) has rank at most \(\min(p,n) < K\), so it is impossible to recover \(K\) consistently.  
Thus the proposed method CSVT attains the fundamental limit of estimability. We should also emphasize that the factor \(\sqrt{K}\) appearing in the lower bound for \(\Delta\) is also optimal. To see this, consider a balanced design with \(\beta = 1\) and suppose the \(K\) centres are placed at the vertices of a regular simplex in \(\mathbb{R}^{K-1}\).  
Then the centred signal matrix \(\widetilde{\mathbf{P}}\) has rank \(K-1\) and its smallest positive singular value is of order \(\Delta\sqrt{n/K}\).  
The noise level is of order \(\sqrt{p} + \sqrt{n}\).  
Matching the signal to the noise therefore forces \(\Delta\) to be at least of order \(\sqrt{K}(\sqrt{p}+\sqrt{n})/\sqrt{n}\), which is exactly the dependence captured by Theorem~\ref{thm:main} (up to constants and the factors \(\kappa\) and \(\beta\)).  
Hence the \(\sqrt{K}\) term is unavoidable in general. Furthermore, when \(K = 1\), the situation is simpler. Then centering removes the entire signal, so \(\widetilde{\mathbf{X}} = \widetilde{\mathbf{E}}\) is pure noise. The same concentration bound guarantees that all singular values of \(\widetilde{\mathbf{X}}\) lie below the threshold \(T\) with probability tending to one, and therefore \(\widehat{K}=1\) automatically. Thus, for the case $K=1$, no separation condition is needed because there is no signal to separate.

In terms of the aspect ratio $p/n$, note that the separation condition \(\Delta \ge 2\sqrt{2}\,\frac{\kappa}{\sqrt{\beta}}\,\sqrt{\frac{K}{n}}\,T\) with \(T = \sqrt{p}+\sqrt{n}+t_n\) takes different explicit forms depending on the relative growth of \(p\) and \(n\).  
When \(n\) is much larger than \(p\) (so that \(p/n \to 0\)), we have \(T \sim \sqrt{n}\) and the condition becomes \(\Delta \gtrsim \frac{\kappa}{\sqrt{\beta}}\sqrt{K}\).  
In this case the required separation does not grow with the sample size, which is natural because a large \(n\) makes the signal easier to detect. When \(p\) is much larger than \(n\) (the high dimensional setting with \(p/n \to \infty\)), we have \(T \sim \sqrt{p}\) and the condition becomes \(\Delta \gtrsim \frac{\kappa}{\sqrt{\beta}}\sqrt{\frac{Kp}{n}}\).  
Thus the required separation grows slowly with the dimension, at rate \(\sqrt{p/n}\). When \(p\) and \(n\) are of the same order (so that \(p/n \to c \in (0,\infty)\)), we have \(T \sim (\sqrt{c}+1)\sqrt{n}\) and the condition simplifies to \(\Delta \gtrsim \frac{\kappa}{\sqrt{\beta}}\sqrt{K(1+\sqrt{c})}\).  
In all regimes, the method remains consistent as long as the separation condition holds. This demonstrates that the proposed estimator works effectively even when the dimension exceeds the sample size, a setting where many classical methods fail. The only requirement is that the signal to noise ratio, as captured by the lower bound on \(\Delta\), remains sufficiently large relative to the dimension and the sample size.
\subsection{Why centering is necessary: counterexamples and discussion}\label{Whycentering}
The proposed CSVT estimator centers the data before thresholding the singular values. One might wonder whether this centering step is truly necessary, or whether a simpler procedure (counting the singular values of the raw data matrix \(\mathbf{X}\) that exceed the same threshold \(T\)) could also consistently estimate the true number of components \(K\). In this subsection, we demonstrate that centering is not merely a convenient device but an essential ingredient for universal consistency.  We first show that without centering the estimator fails already for the trivial case \(K=1\).  We then discuss the general advantages of centering for \(K\ge 2\), namely translation invariance and threshold robustness, which hold even when the centre matrix has full column rank.  Finally, we present a more subtle counterexample where the centre matrix is rank deficient (so that the condition number \(\kappa\) is infinite) and the uncentered method breaks down completely, while the centred estimator remains consistent.  These examples highlight the fundamental role of centering.
\subsubsection{Centering is crucial for \(K=1\)}
Consider the case \(K=1\). Then, we have
\[
\mathbf{X} = \boldsymbol{\mu}_1 \mathbf{1}_n^\top + \mathbf{E},
\]
where \(\boldsymbol{\mu}_1\in\mathbb{R}^p\) is the only centre and \(\mathbf{E}\) has i.i.d. \(\mathcal{N}(0,1)\) entries.  
Let \(T = \sqrt{p}+\sqrt{n}+t_n\) with \(t_n\to\infty\) and \(t_n = o(\sqrt{p}+\sqrt{n})\). The following counterexample considers the degenerate situation where the signal is exactly zero (\(\mu_1 = \mathbf{0}\)). Although this violates the full‑rank condition of Theorem \ref{thm:main}, it serves as a stress test: the uncentered estimator fails, while the centered estimator remains consistent. The advantage of centering is that it works for every \(\mu_1\), including this boundary case.
\begin{itemize}
  \item \noindent\textbf{Without centering.} A naive estimator would be $\widehat{K}_{\text{raw}} = |\{i : \sigma_i(\mathbf{X}) > T\}|$. By Lemma \ref{lem:noise}, $\|\mathbf{E}\| \le T$ with probability tending to $1$.  
If $\boldsymbol{\mu}_1 = \mathbf{0}$ (i.e., the data already have zero mean), then $\mathbf{X} = \mathbf{E}$, and consequently all singular values of $\mathbf{X}$ satisfy $\sigma_i(\mathbf{X}) \le T$ with high probability. Hence $\widehat{K}_{\text{raw}} = 0$ w.h.p., which is incorrect because the true number of components is $1$. This counterexample shows that the uncentered method is not universally consistent.
\item \noindent\textbf{After centering.} Let $\mathbf{H} = \mathbf{I}_n - \frac{1}{n}\mathbf{1}_n\mathbf{1}_n^\top$ be the centering matrix, satisfying $\mathbf{1}_n^\top\mathbf{H} = \mathbf{0}$. The centred data matrix is $\widetilde{\mathbf{X}} = \mathbf{X}\mathbf{H}$. For $K=1$, we have
\[
\widetilde{\mathbf{X}} = (\boldsymbol{\mu}_1 \mathbf{1}_n^\top + \mathbf{E})\mathbf{H} = \boldsymbol{\mu}_1 \mathbf{1}_n^\top \mathbf{H} + \mathbf{E}\mathbf{H} = \mathbf{E}\mathbf{H},
\]
because $\mathbf{1}_n^\top \mathbf{H} = \mathbf{0}$. Hence, we have $\|\widetilde{\mathbf{X}}\| \le \|\mathbf{E}\|$ and, by Lemma \ref{lem:noise}, $\|\widetilde{\mathbf{X}}\| \le T$ with probability tending to $1$. Therefore, all singular values of $\widetilde{\mathbf{X}}$ are below $T$ with high probability, so $r = |\{i:\sigma_i(\widetilde{\mathbf{X}}) > T\}| = 0$ w.h.p., and the CSVT estimator returns $\widehat{K} = r+1 = 1$. This conclusion holds \textbf{for every} $\boldsymbol{\mu}_1$ (including $\boldsymbol{\mu}_1 = \mathbf{0}$), without any signal strength condition.
\end{itemize}

Thus, centering makes the estimator universally consistent for \(K=1\), whereas the uncentered version fails not only in the extreme case \(\boldsymbol{\mu}_1 = \mathbf{0}\) but also whenever \(\|\boldsymbol{\mu}_1\|\) is sufficiently small (i.e., \(\sqrt{n}\|\boldsymbol{\mu}_1\|\) is below the noise threshold. A toy example can be found in Remark \ref{rem:centering_K1_demo}.). This fundamental advantage justifies why centering is “the key first step” in our CSVT algorithm.

\begin{rem}\label{rem:centering_K1_demo}
The simulation results displayed in Figure~\ref{fig:centering_K1_demo} (with $n=100$, $p=20$, $\boldsymbol{\mu}_1 =0.1\cdot\mathbf{1}_p$, hence $\|\boldsymbol{\mu}_1 \|\approx0.4472$) clearly demonstrate the necessity of centering when $K=1$.  
The threshold is set to $T=\sqrt{p}+\sqrt{n}+\log n \approx 4.472+10+4.605 \approx 19.077$. From this figure, we clearly observe that:
\begin{itemize}
  \item \texttt{Left panel (uncentered data $\mathbf{X}$)}: All singular values of $\mathbf{X}$ lie below the dashed red line $T$. The largest singular value $\sigma_1(\mathbf{X})$ is approximately $14.5$, which is dominated by the noise because the signal strength $\sqrt{n}\|\boldsymbol{\mu}_1 \| \approx 4.47$ is far smaller than the typical noise level $\sqrt{p}+\sqrt{n}\approx14.47$.  Since no singular value exceeds $T$, the raw count $|\{i:\sigma_i(\mathbf{X})>T\}|$ equals $0$, leading to the incorrect estimate $\widehat{K}_{\text{raw}} = 0$ (instead of the true $K=1$).
\item \texttt{Right panel (centered data $\widetilde{\mathbf{X}} = \mathbf{X}\mathbf{H}$):} Centering removes the constant shift $\boldsymbol{\mu}_1 \mathbf{1}_n^\top$ completely because $\mathbf{1}_n^\top\mathbf{H}=\mathbf{0}$. Hence $\widetilde{\mathbf{X}} = \mathbf{E}\mathbf{H}$, which is pure noise. All singular values of $\widetilde{\mathbf{X}}$ are also bounded by $\|\mathbf{E}\| \le T$ with high probability, so again $r = |\{ i : \widehat{\sigma}_i > T \}|=0$. However, the CSVT estimator adds $1$ to this count, yielding $\widehat{K}= r+1 = 1$, which correctly recovers the true number of components.
\end{itemize}

The figure therefore validates the theoretical claim: for $K=1$, centering guarantees consistent estimation for \emph{any} $\boldsymbol{\mu}_1 $ (including $\boldsymbol{\mu}_1 =\mathbf{0}$), whereas the uncentered approach fails whenever the signal is too weak to push $\sigma_1(\mathbf{X})$ above the noise threshold. 

\begin{figure}[!htbp]
\centering
\includegraphics[width=0.9\textwidth]{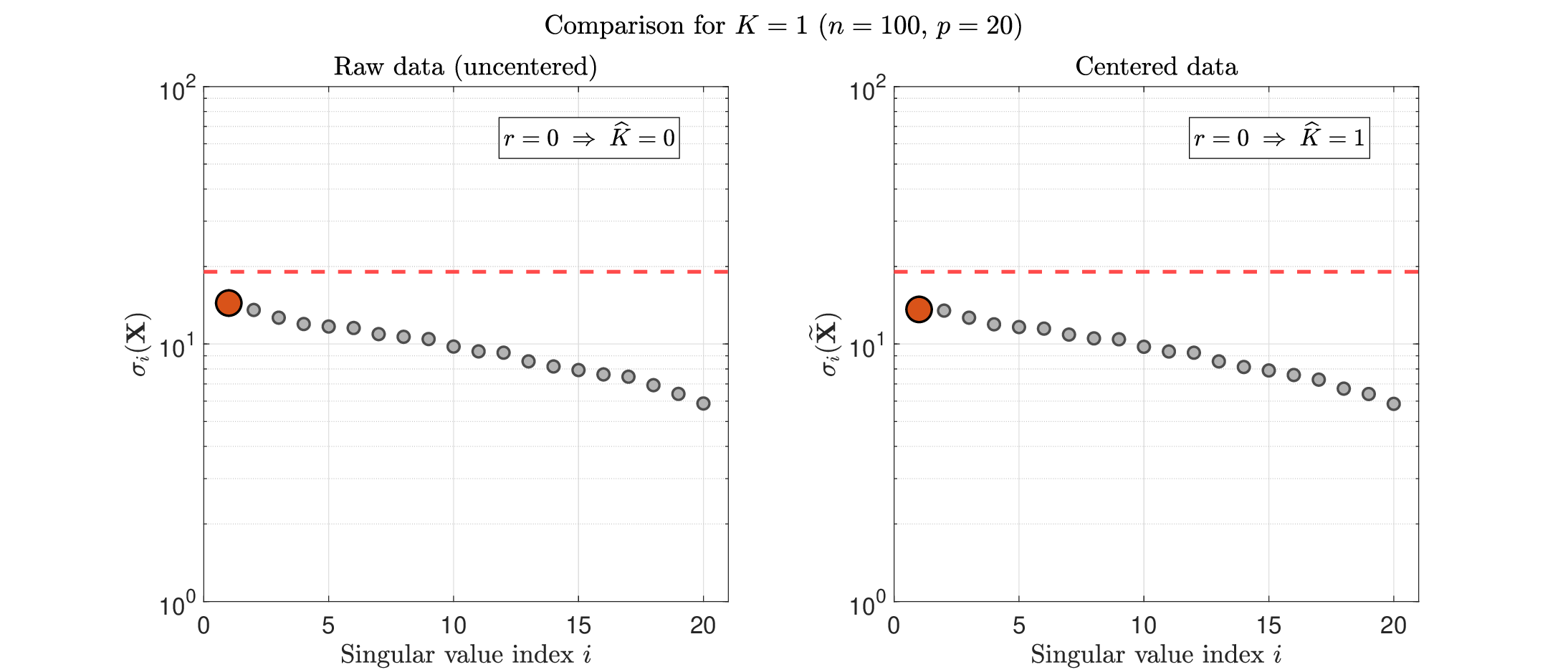}
\caption{Comparison of singular values for uncentered (left) and centered (right) data when $K=1$. 
The dashed red line marks the threshold $T=\sqrt{p}+\sqrt{n}+\log n$. 
In the left panel, all singular values are below $T$, so the raw estimate is $0$ (incorrect). 
In the right panel, after centering, all singular values are also below $T$, but adding $1$ yields the correct estimate $\widehat{K}=1$.}
\label{fig:centering_K1_demo}
\end{figure}
\end{rem}
\subsubsection{General advantages of centering for \(K\ge 2\).}
Although the worst-case theoretical lower bounds for \(\sigma_K(\mathbf{P})\) (without centering) and \(\sigma_{K-1}(\widetilde{\mathbf{P}})\) (with centering) are identical when $\kappa$ is finite by Lemmas \ref{lem:signal} and \ref{lem:uncentered_signal}, centering offers a fundamental advantage when \(K\ge 2\): translation invariance and threshold robustness.  
For any \(\mathbf{c}\in\mathbb{R}^p\), we have
\[
(\mathbf{X}+\mathbf{c}\mathbf{1}_n^\top)\mathbf{H} = \mathbf{X}\mathbf{H} = \widetilde{\mathbf{X}},
\]
because \(\mathbf{1}_n^\top\mathbf{H}=0\). Hence, the centred data matrix (and therefore its singular values) is invariant under global shifts of all data points. As a consequence, the singular values of \(\widetilde{\mathbf{P}}\) depend only on the differences \(\boldsymbol{\mu}_k-\boldsymbol{\mu}_\ell\), not on the global mean \(\boldsymbol{\mu}_0\). The threshold \(T = \sqrt{p}+\sqrt{n}+t_n\) depends only on the noise level, so the same \(T\) works uniformly for any signal magnitude.  

Without centering, \(\sigma_K(\mathbf{P})\) can be made arbitrarily small even when the centres are well separated, either by adding a large common shift (which does not affect the differences but increases the global mean) or by making the centre matrix nearly rank deficient.  The following explicit counterexample for \(K=2\) takes the extreme case where the centres are collinear, so that \(\mathbf{M}\) has rank \(1\) and the condition number \(\kappa\) is infinite.  In this setting, the uncentered method fails completely, while the centred estimator continues to work thanks to translation invariance.
\begin{lem}[Uncentered case can fail even with fixed \(\Delta\)]\label{pathology}
Let \(K=2\), \(p\ge 2\), \(n\ge 2\). Choose centres
\[
\boldsymbol{\mu}_1 = \mathbf{c} + \mathbf{v}, \qquad \boldsymbol{\mu}_2 = \mathbf{c} - \mathbf{v},
\]
with \(\mathbf{c}=t\mathbf{v}\) for some \(t>0\), and \(\|\mathbf{v}\| = \Delta/2\) where \(\Delta = \|\boldsymbol{\mu}_1-\boldsymbol{\mu}_2\|>0\) is fixed.  
Let the sample sizes be \(n_1,n_2\ge 1\) with \(n=n_1+n_2\). Define \(\mathbf{w}^\top\in\mathbb{R}^{1\times n}\) such that \(\mathbf{w}_i = +1\) if observation \(i\) belongs to cluster \(1\) and \(\mathbf{w}_i = -1\) if it belongs to cluster \(2\).  
The signal matrix is \(\mathbf{P} = \boldsymbol{\mu}_1\mathbf{z}_1^\top+ \boldsymbol{\mu}_2\mathbf{z}_2^\top = \mathbf{c}\mathbf{1}_n^\top + \mathbf{v}\mathbf{w}^\top\).  
Let \(\mathbf{X} = \mathbf{P} + \mathbf{E}\) where \(\mathbf{E}\) has i.i.d. \(\mathcal{N}(0,1)\) entries. Define the threshold \(T = \sqrt{p}+\sqrt{n}+t_n\) with \(t_n\to\infty\), \(t_n=o(\sqrt{p}+\sqrt{n})\).  
Consider the uncentered estimator \(\widehat{K}_{\text{raw}} = |\{i:\sigma_i(\mathbf{X}) > T\}|\). Then for any sufficiently large \(t\), with probability tending to \(1\), we have
\[
\widehat{K}_{\text{raw}} = 1 \quad\text{while}\quad K=2.
\]

Thus, the uncentered estimator is not consistent, even though the separation \(\Delta\) is fixed and positive.
\end{lem}

The centred estimator \(\widehat{K} = |\{i:\sigma_i(\widetilde{\mathbf{X}})=\sigma_i(\mathbf{X}\mathbf{H}) > T\}|+1\) with \(\mathbf{H} = \mathbf{I}_n - \frac{1}{n}\mathbf{1}_n\mathbf{1}_n^\top\) in Algorithm~\ref{alg:CSVT} does not suffer from the pathology in Lemma~\ref{pathology}.  
Indeed, under the same settings of Lemma~\ref{pathology}, we have
\[
\widetilde{\mathbf{P}} = \mathbf{P}\mathbf{H} = \mathbf{v}\mathbf{w}^\top\mathbf{H},
\]
which is a rank-1 outer product. Hence, we have
\[
\sigma_1(\widetilde{\mathbf{P}}) = \|\mathbf{v}\| \;\|\mathbf{w}^\top\mathbf{H}\|=\frac{\Delta}{2}\;\|\mathbf{w}^\top\mathbf{H}\|
\]
because \(\|\mathbf{v}\| = \Delta/2\). For \(\|\mathbf{w}^\top\mathbf{H}\|\), we have
\[
\mathbf{w}^\top\mathbf{H} = \mathbf{w}^\top - \bar w\mathbf{1}_n^\top,\qquad \bar w = \frac{n_1-n_2}{n},
\]
which gives
\[
\|\mathbf{w}^\top\mathbf{H}\|^2 = \sum_i (\mathbf{w}_i-\bar w)^2 = \sum_i \mathbf{w}_i^2 - n\bar w^2 = n - n\left(\frac{n_1-n_2}{n}\right)^2 = \frac{4n_1n_2}{n}.
\]

Thus, we get
\[
\|\mathbf{w}^\top\mathbf{H}\| = \sqrt{\frac{4n_1n_2}{n}}.
\]

Finally, we obtain
\[
\sigma_1(\widetilde{\mathbf{P}}) = \frac{\Delta}{2} \cdot \sqrt{\frac{4n_1n_2}{n}} = \Delta\sqrt{\frac{n_1n_2}{n}},
\]
which is independent of \(t\) and $\kappa$. Moreover, \(\sigma_2(\widetilde{\mathbf{P}})=0\). Provided that \(\Delta\) satisfies the separation condition of Theorem~\ref{thm:main} (which here simplifies to \(\sigma_1(\widetilde{\mathbf{P}})= \Delta\sqrt{\frac{n_1n_2}{n}}\gtrsim T\) because \(\beta\) and \(\kappa\) are not applicable), the centred method correctly identifies \(K=2\) with high probability.  Thus, centering removes the dependence on the global mean and guarantees consistent estimation even in this extreme rank‑deficient scenario.
\section{Simulation experiments}
\label{sec:experiments}

We conduct simulation studies to evaluate the performance of the proposed CSVT estimator. Unless specified, the data are generated from the Gaussian mixture model defined in Equation \eqref{def:GMM}, where the noise matrix \(\mathbf{E}\) has i.i.d. standard normal entries. All simulations are repeated \(100\) times independently. In each repetition, we compute the estimated number of components \(\widehat{K}\) using Algorithm~\ref{alg:CSVT} with the default threshold \(T = \sqrt{p}+\sqrt{n}+\log n\). In each repetition, we record whether the estimated number of components \(\widehat{K}\) equals the true number of components $K$, and we will report the proportion of runs where this occurs (i.e., the accuracy rate). The experiments are designed to cover a wide range of challenging settings, including extreme sample sizes, high dimensionality, large numbers of clusters, and severe cluster imbalance. We emphasize that we do not compare our CSVT estimator with existing methods in this section, as to the best of our knowledge, no existing method can reliably estimate the number of components under the extreme settings considered here. These settings include, for example, sample sizes or dimensions as large as ten million, numbers of components reaching the maximal possible value $\min(n,p)$, and cases where the smallest cluster contains only twenty observations out of one million. Classical information criteria such as AIC and BIC rely on likelihood evaluations that become numerically unstable or ill-defined in high dimensions. The expectation-maximization algorithm, a standard tool for fitting Gaussian mixtures, often fails to converge or runs out of memory when $n$ or $p$ reaches millions, and $K$ grows large. Our goal is therefore not to benchmark against methods that cannot be meaningfully executed here. Instead, we focus on verifying that the proposed CSVT estimator works as theoretically predicted in the very scenarios where existing approaches are not applicable.

\subsection{Data generation and parameter choices}\label{simdatagen}

To isolate the effect of each model parameter, we construct the centre matrix \(\mathbf{M}\) to be orthogonal. Specifically, we set
\[
\mathbf{M} = \frac{\Delta}{\sqrt{2}} \mathbf{Q},
\]
where \(\mathbf{Q}\) is a \(p\times K\) matrix with orthonormal columns. Then for any two distinct centres we have \(\|\boldsymbol{\mu}_k - \boldsymbol{\mu}_\ell\| = \Delta\), and the condition number is \(\kappa = 1\). This choice satisfies the assumptions of Theorem~\ref{thm:main} in the simplest manner.

The cluster sizes are determined by the imbalance parameter \(\beta\). Let \(n_k\) be the size of the \(k\)th cluster. We set  
\[
n_{\min} = \max\bigl(1,\; \lfloor \beta n / K \rfloor \bigr).
\]  

We first assign \(n_{\min}\) observations to each of the \(K\) clusters. Then we select one cluster uniformly at random to keep its size at \(n_{\min}\) (the smallest cluster), and distribute the remaining \(n - K n_{\min}\) observations as evenly as possible among the other \(K-1\) clusters. For each observation \(i\) belonging to cluster \(k\), we draw \(\mathbf{x}_i \sim \mathcal{N}(\boldsymbol{\mu}_k, \mathbf{I}_p)\) independently.

Unless specified, to ensure that the separation condition of Theorem~\ref{thm:main} is satisfied, we choose
\[
\Delta = 2\sqrt{2}\,\frac{1}{\sqrt{\beta}}\,\sqrt{\frac{K}{n}}\;T,
\]
i.e. we take \(\Delta\) equal to the theoretical lower bound. For experiments where \(n, p, K\) or \(\beta\) vary, the value of \(\Delta\) is updated accordingly using the corresponding \(n\), \(p\), \(K\), \(\beta\) and \(T\).

\subsection{Experiment designs}

We organise the simulations into five parts, each focusing on a different aspect of the estimator.

\subsubsection{Effect of sample size and dimension}

We consider three regimes: sample‑rich (\(n\gg p\)), high‑dimensional (\(p\gg n\)), and balanced (\(n=p\)).

\begin{itemize}
    \item \textbf{Sample‑rich regime:} Fix the dimension \(p = 100\) and the number of components \(K = 99\). Set the imbalance parameter \(\beta = 0.1\) (unbalanced clusters). Let the sample size \(n\) take ten values from \(1\times10^6\) to \(1\times10^7\) in increments of \(1\times10^6\), i.e. \(n = 1\times10^6, 2\times10^6, \dots, 1\times10^7\).
    \item \textbf{High‑dimensional regime:} Fix the sample size \(n = 100\) and the number of components \(K = 10\). Set \(\beta = 1\) (balanced clusters). Let the dimension \(p\) take ten values from \(1\times10^6\) to \(1\times10^7\) in increments of \(1\times10^6\), i.e. \(p = 1\times10^6, 2\times10^6, \dots, 1\times10^7\).
    \item \textbf{Balanced regime:} Let \(n = p\) and let this common value vary from \(1000\) to \(10000\) in steps of \(1000\). Set \(K = n/100\) (so that the number of components grows with the sample size) and \(\beta = 0.1\) (unbalanced clusters).
\end{itemize}
\begin{figure}[!htbp]
\centering
\resizebox{\columnwidth}{!}{
{\includegraphics[width=3\textwidth]{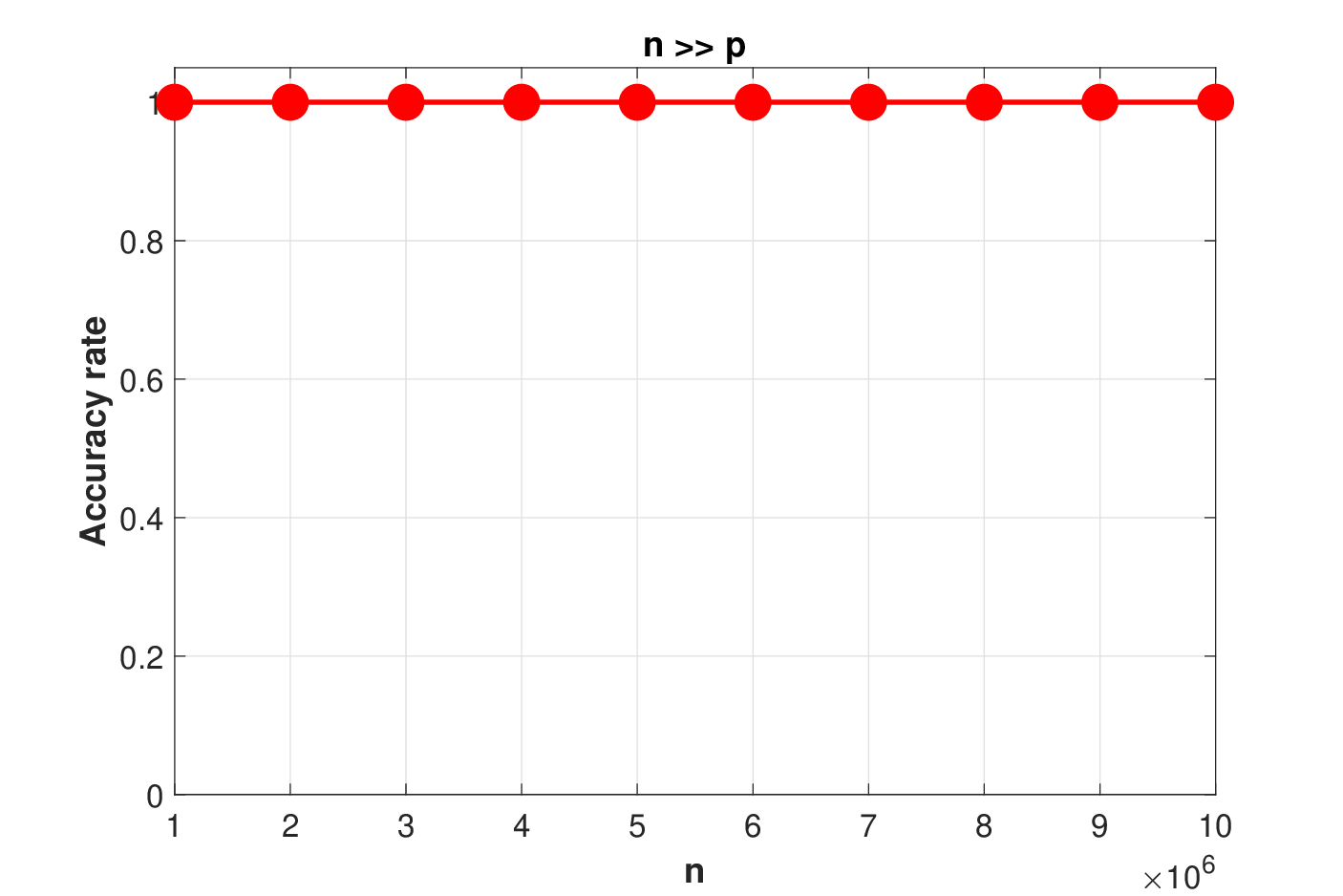}}
{\includegraphics[width=3\textwidth]{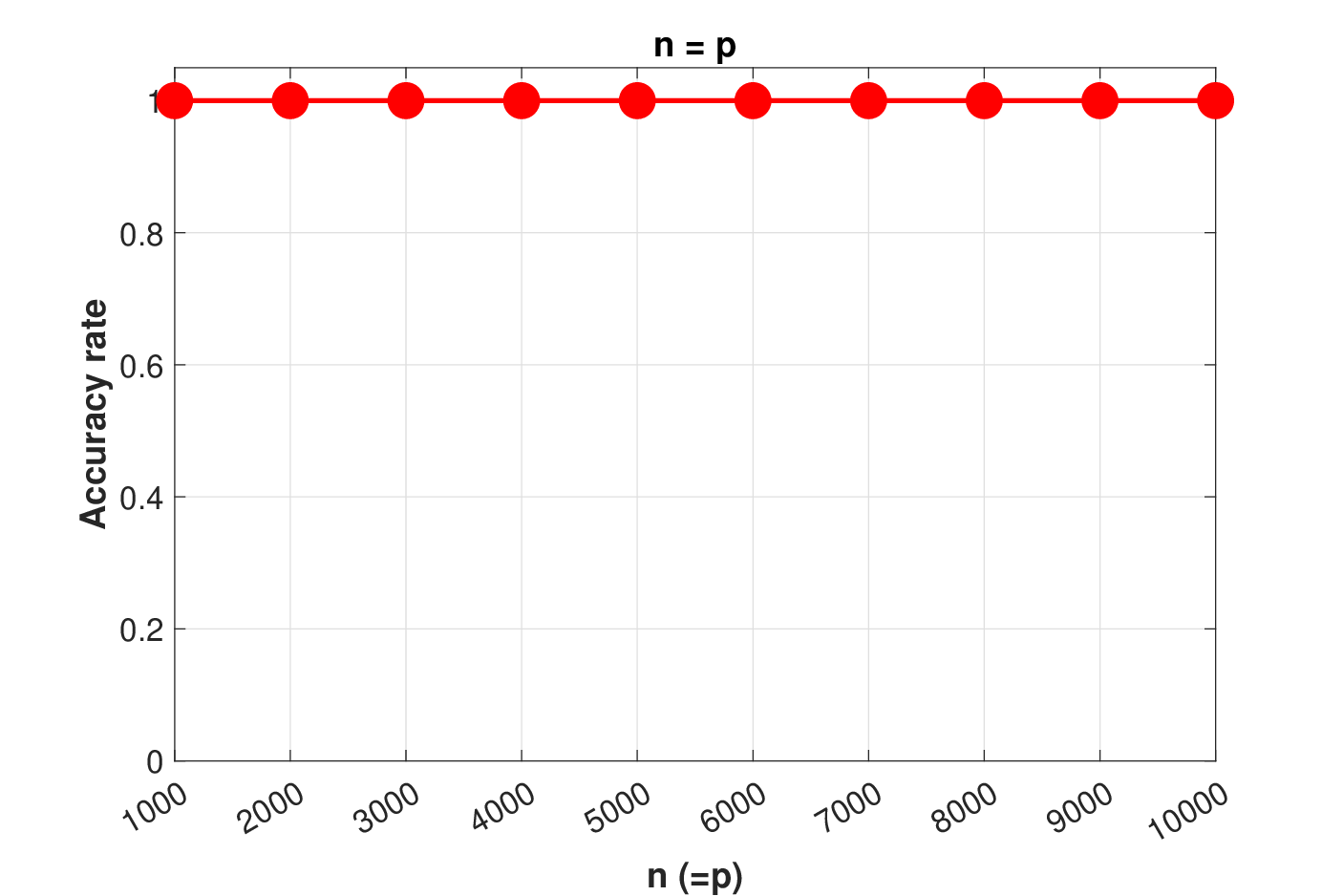}}
{\includegraphics[width=3\textwidth]{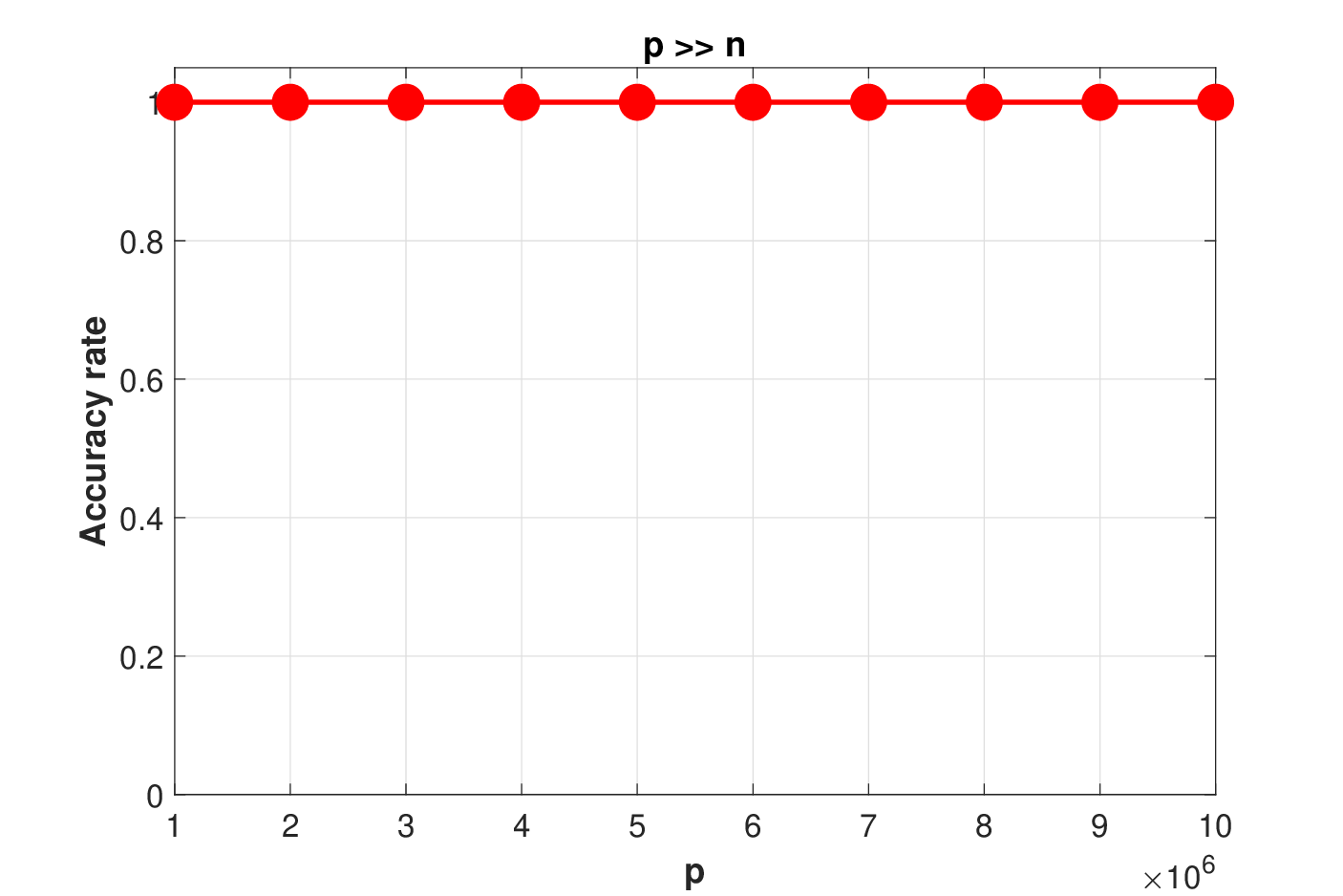}}
}
\resizebox{\columnwidth}{!}{
{\includegraphics[width=3\textwidth]{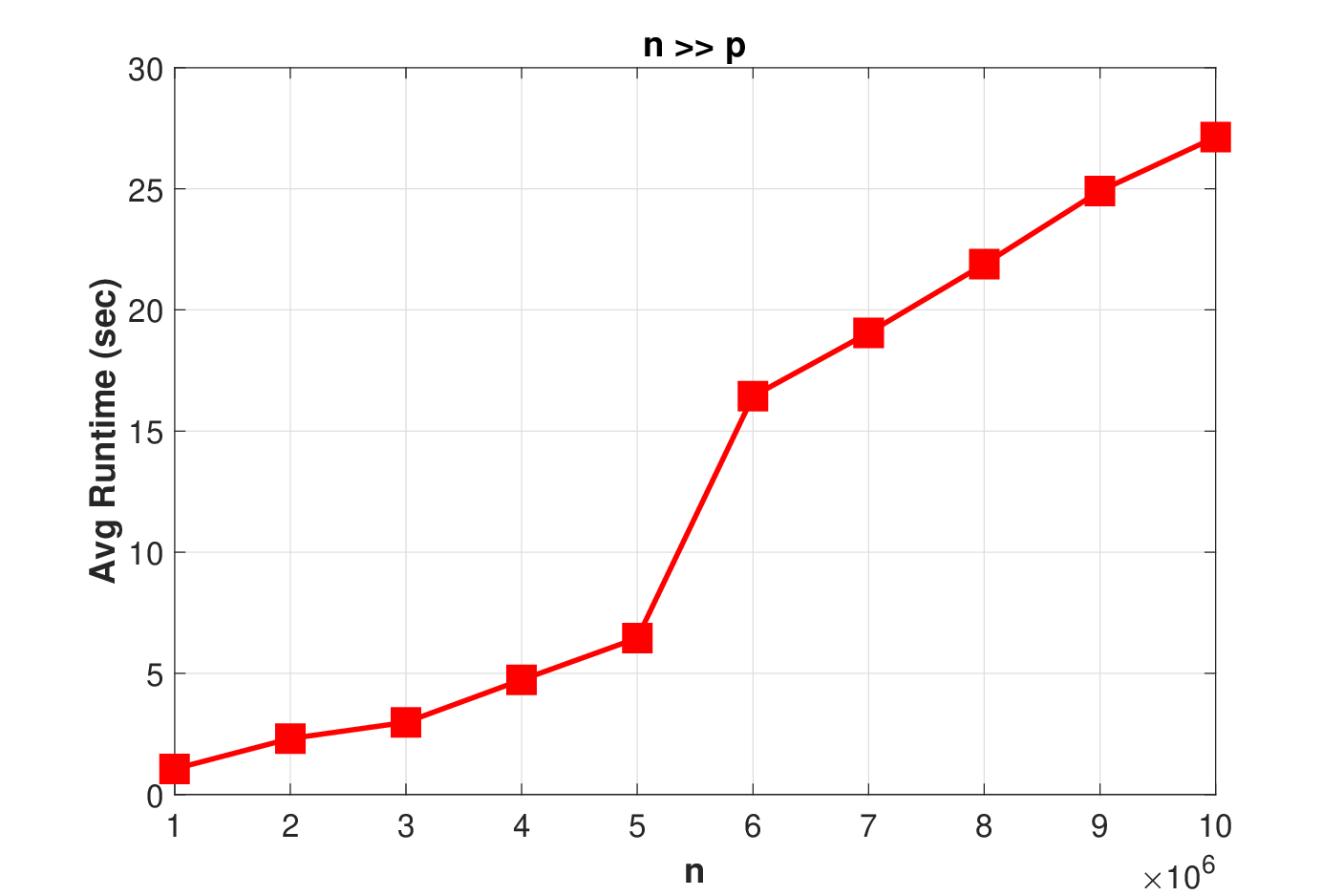}}
{\includegraphics[width=3\textwidth]{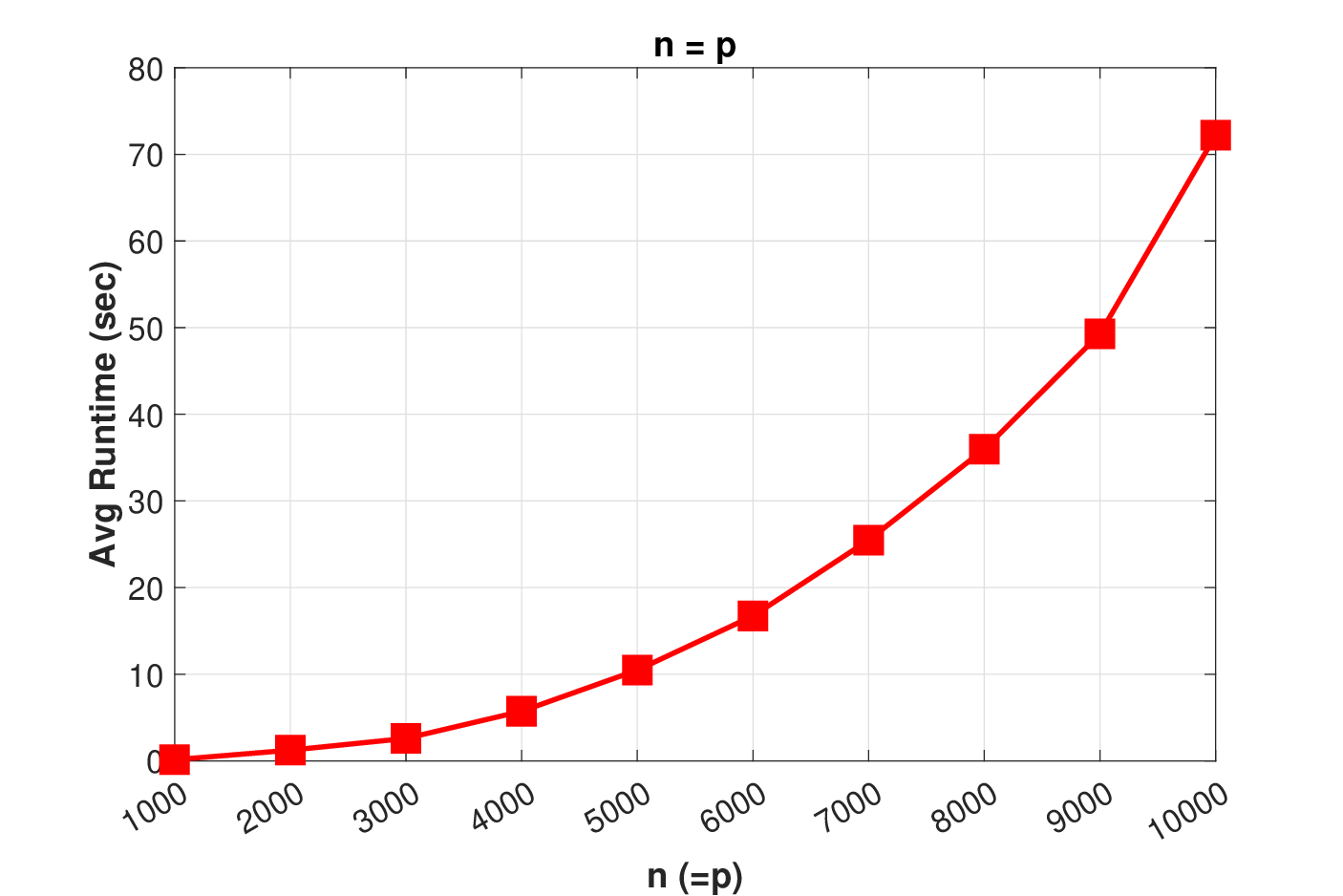}}
{\includegraphics[width=3\textwidth]{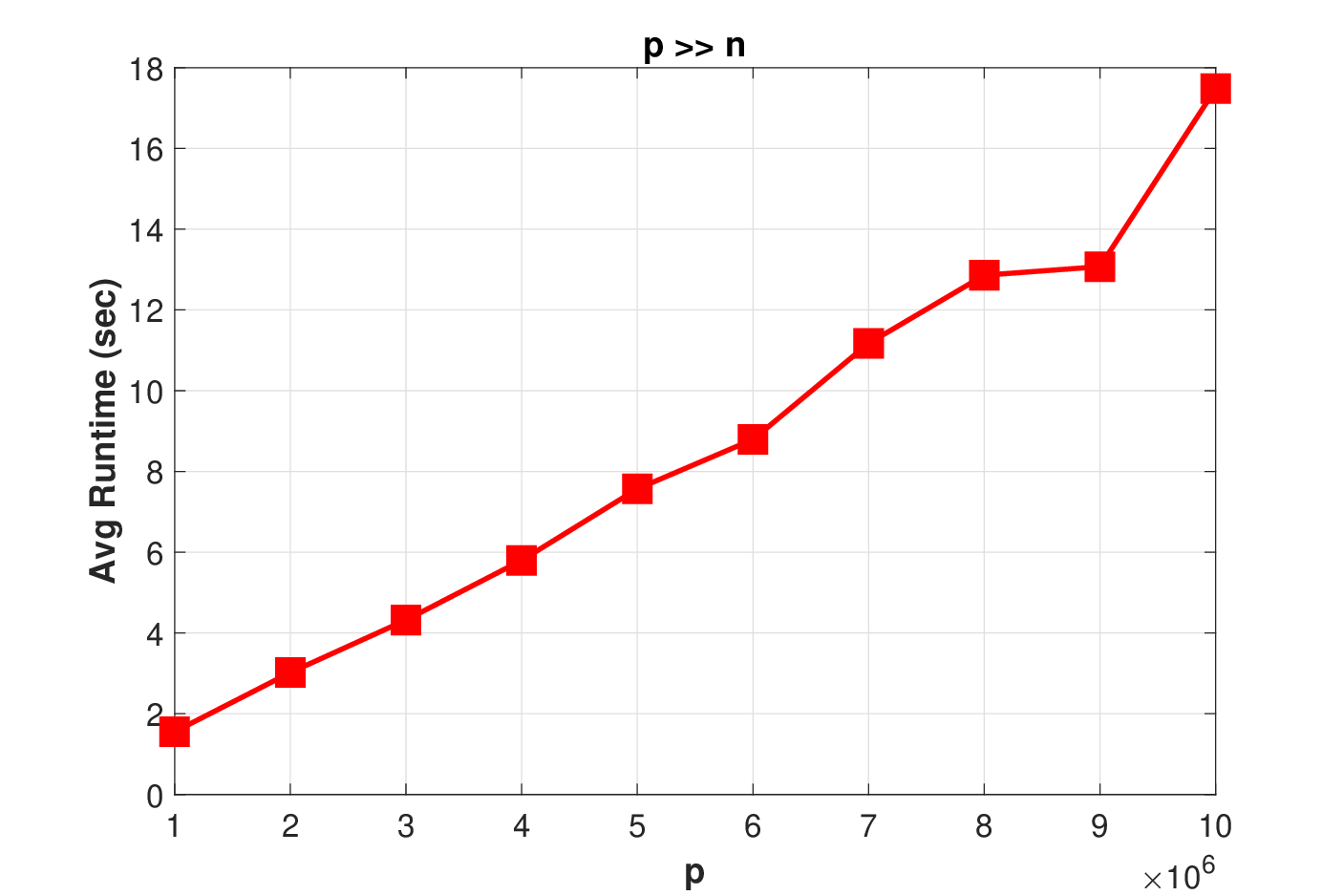}}
}
\caption{Numerical results of Experiment 1.}
\label{fig:ex1} 
\end{figure}

For each combination of parameters, more than the accuracy rate, we also report the average computational time. Figure~\ref{fig:ex1} reports all numerical results across the three regimes. The estimator achieves perfect accuracy in every repetition. The running time is remarkably low even at extreme scales. In the sample-rich regime with $p=100$ and $n$ up to ten million, as well as the high-dimensional regime with $n=100$ and $p$ up to ten million, the algorithm terminates in less than 30 seconds. Classical approaches, such as the expectation-maximization algorithm, become computationally infeasible at this scale. In the balanced regime where $n=p$ grows from one thousand to ten thousand and $K=n/100$ (so the maximum $K$ is 100), the procedure completes in less than 80 seconds. These results confirm that the proposed CSVT estimator is both statistically consistent and computationally practical across all tested settings, including the challenging high dimensional case where the dimension far exceeds the sample size.

\begin{rem}\label{rem:simulated_svd_demo}
Figure~\ref{fig:simulated_svd} illustrates the behaviour of the proposed CSVT estimator for a typical simulation setting with $K=5$.  
Data are generated according to settings in Section \ref{simdatagen} with $n=1000$, $p=20$, $K=5$, and $\beta=1$.  
After centering the data matrix $\mathbf{X}$, we compute the singular values $\sigma_i(\widetilde{\mathbf{X}})$ of $\widetilde{\mathbf{X}}=\mathbf{X}\mathbf{H}$ and the threshold $T=\sqrt{p}+\sqrt{n}+\log n$. This figure shows that the first $K-1=4$ singular values (highlighted in orange) are all well above the red dashed threshold $T$, while the $5$th singular value (blue diamond) lies clearly below $T$. All remaining singular values (grey circles) also stay below $T$. Consequently, the number of singular values exceeding $T$ is $r=4$, and the CSVT estimator returns $\widehat{K}=r+1=5$, which matches the true number of components. This toy example confirms the theoretical guarantee of Theorem~\ref{thm:main}: under the prescribed separation condition, the centred singular value spectrum exhibits a clean gap between the $(K-1)$-th and $K$-th singular values, allowing consistent estimation of $K$.

\begin{figure}[!htbp]
\centering
\includegraphics[width=1.1\textwidth]{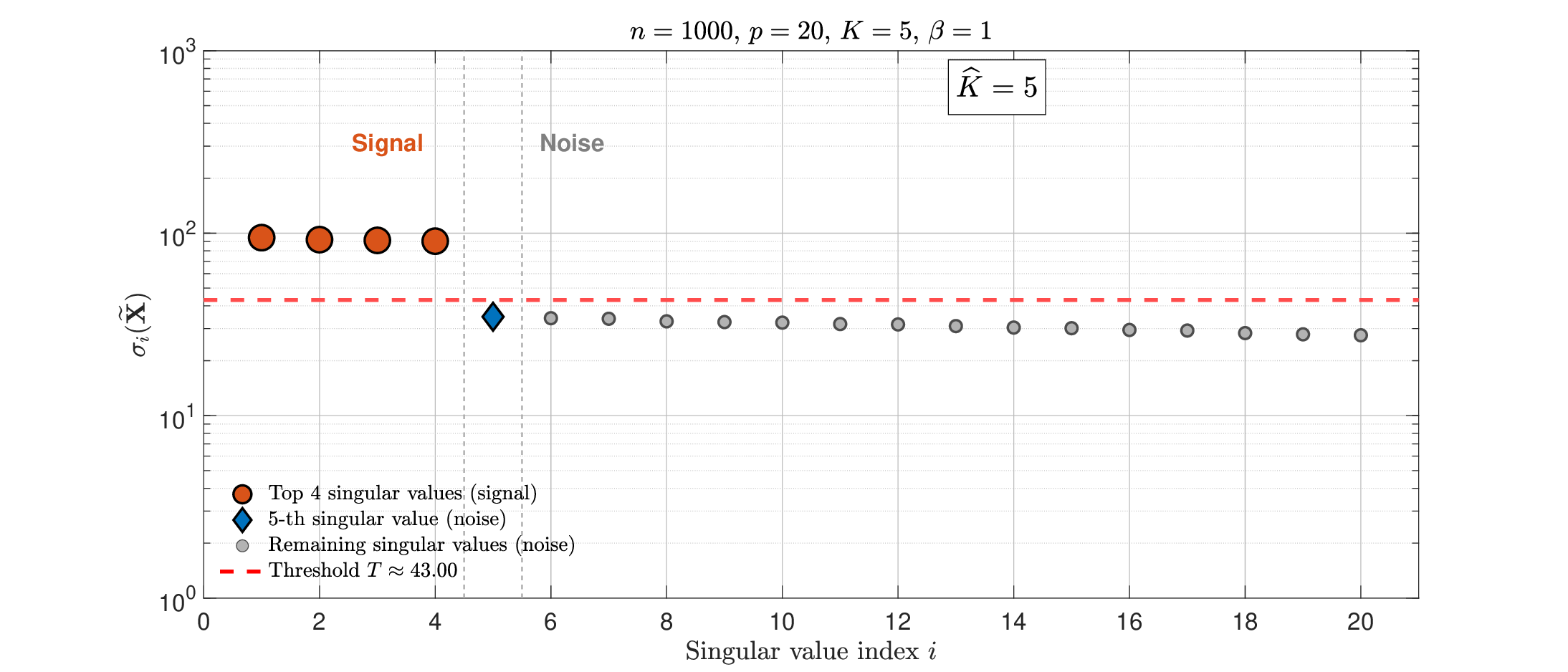}
\caption{Singular values of the centred data matrix $\widetilde{\mathbf{X}}$ for $n=1000$, $p=20$, $K=5$, balanced clusters ($\beta=1$), and $\Delta$ chosen at the theoretical lower bound. 
The first $K-1=4$ singular values (orange circles) exceed the threshold $T$ (red dashed line), while the $5$th singular value (blue diamond) and all subsequent ones lie below $T$, leading to $\widehat{K}=5$.}
\label{fig:simulated_svd}
\end{figure}
\end{rem}
\subsubsection{Growth of the number of components}
We examine how the estimator behaves when the number of clusters \(K\) increases such that \(K\) ranges in $\{1,10,20,\ldots,100\}$ while the total sample size and dimension are held fixed. Two distinct scenarios are considered.

\begin{itemize}
    \item \textbf{Sample‑rich  regime:} Fix \(n = 10^6\), \(p = 100\), and \(\beta = 0.1\). 
    \item \textbf{High‑dimensional  regime:} Fix \(n = 100\), \(p = 10^6\), and \(\beta = 1\).
\end{itemize}
\begin{figure}[!htbp]
\centering
\resizebox{\columnwidth}{!}{
{\includegraphics[width=3\textwidth]{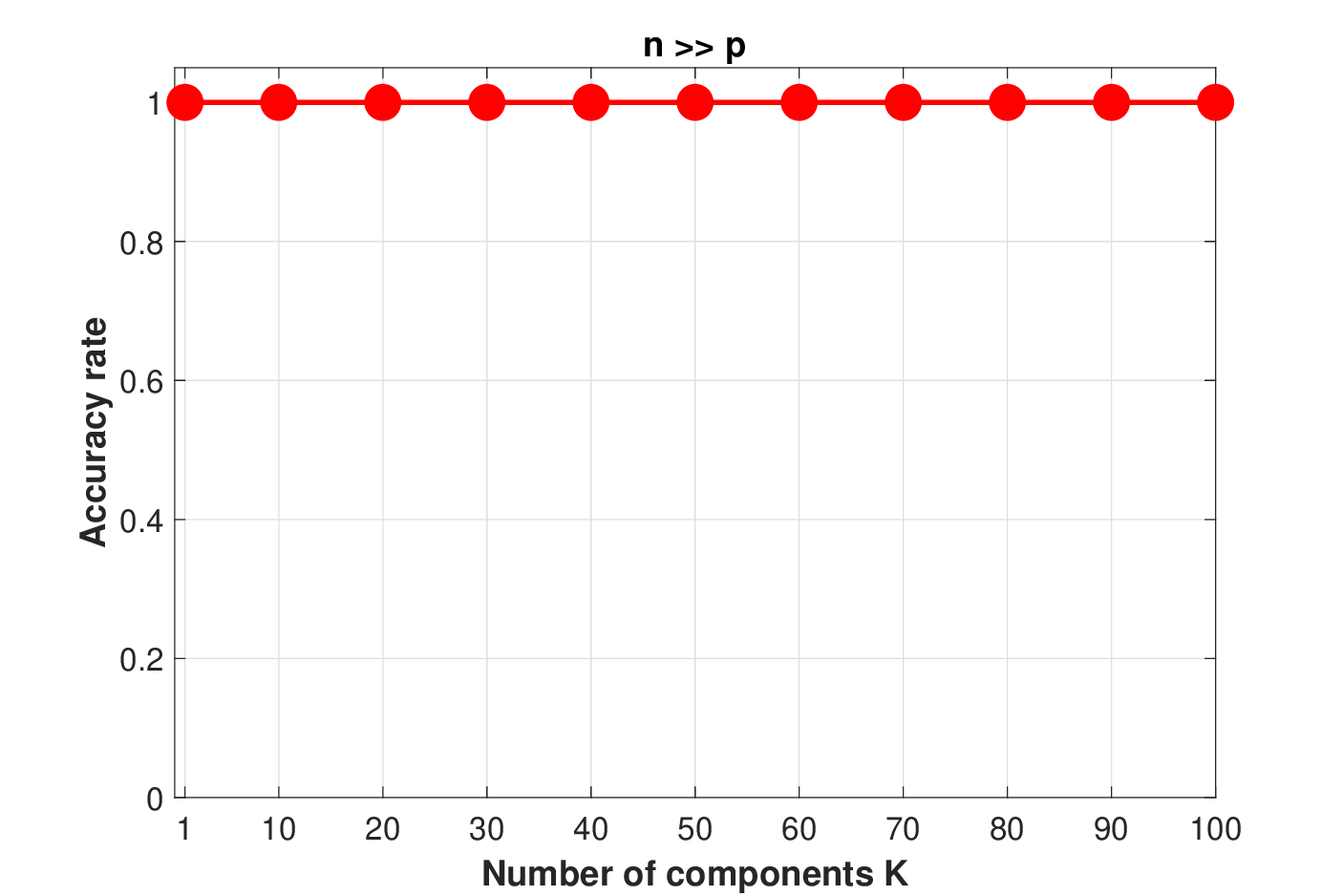}}
{\includegraphics[width=3\textwidth]{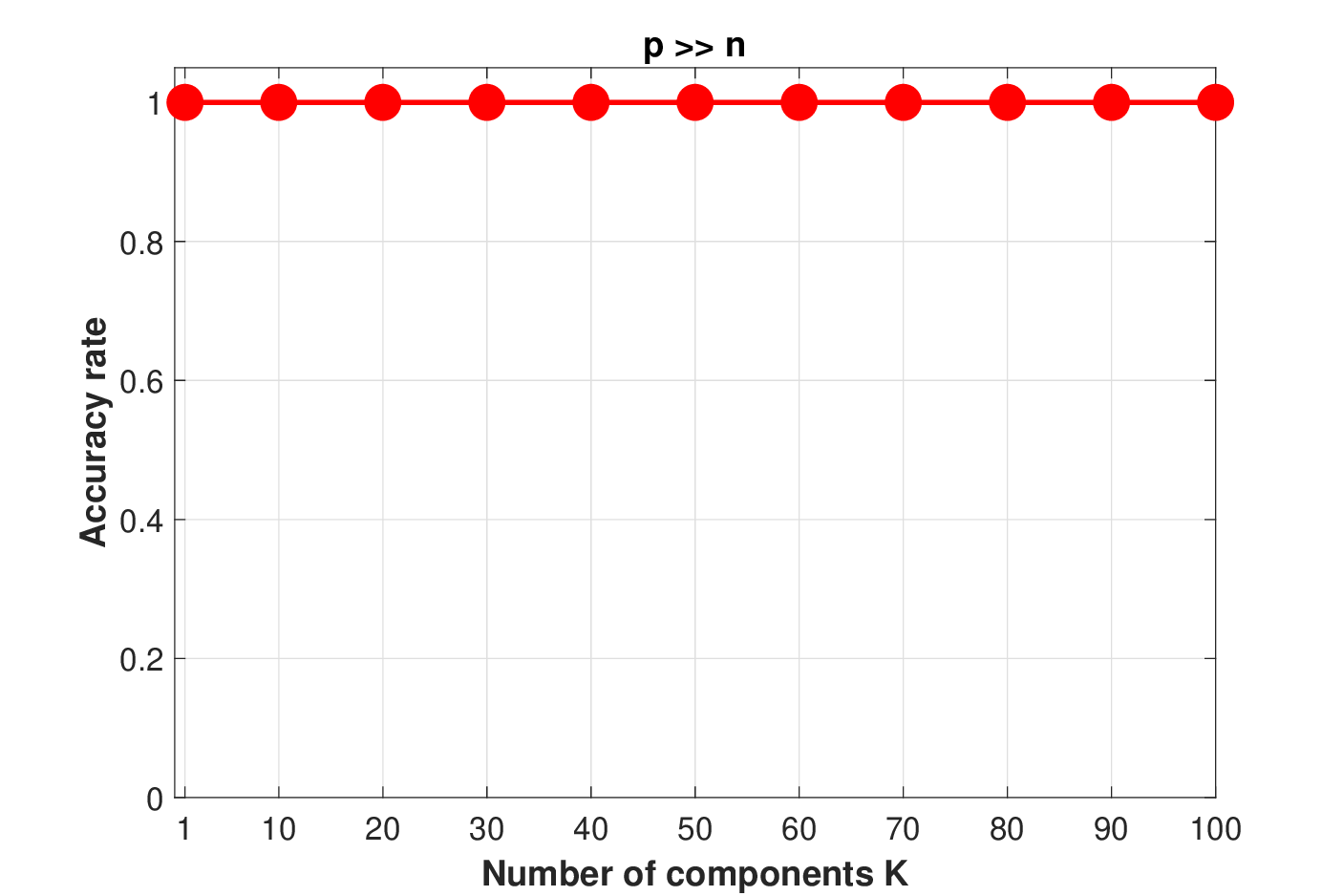}}
}
\caption{Numerical results of Experiment 2.}
\label{fig:ex2} 
\end{figure}

In both scenarios, the maximum value of \(K\) is \(100\), which equals \(\min(p,n)\) in both scenarios, thus approaching the maximal possible value. Figure~\ref{fig:ex2} shows the performance of the CSVT estimator as the number of components $K$ increases from one to one hundred. The estimator recovers the true $K$ perfectly in all cases. This holds for both the sample rich scenario with $n=10^6$ and $p=100$ and the high dimensional scenario with $n=100$ and $p=10^6$. The maximum value $K=100$ reaches the maximal possible value $\min(p,n)$. Even at this boundary, the method continues to work without any loss in accuracy. These results demonstrate that the CSVT estimator reliably handles a large number of clusters up to the fundamental limit imposed by the data matrix dimensions.

\subsubsection{Effect of cluster imbalance}
We investigate the sensitivity of CSVT to severe imbalance among the cluster sizes. Fix \(n = 10^6\), \(p = 200\), and \(K = 50\). Let the imbalance parameter \(\beta\) take values from \(0.001\) to \(0.01\) in steps of \(0.001\), i.e. \(\beta = 0.001, 0.002, \dots, 0.01\). At the most extreme setting \(\beta = 0.001\), the smallest cluster contains only \(n_{\min} = 0.001 \times 10^6 / 50 = 20\) observations.

Figure~\ref{fig:ex3} reports the numerical results of this experiment. The CSVT estimator achieves perfect accuracy across all values of the imbalance parameter $\beta$. Even when the smallest cluster contains only twenty observations, an extremely tiny fraction compared to the total sample size of one million, the method still recovers the true number of components correctly. This demonstrates that the estimator is highly robust to severe imbalance among cluster sizes.
\begin{figure}[!htbp]
\centering
{\includegraphics[width=0.5\textwidth]{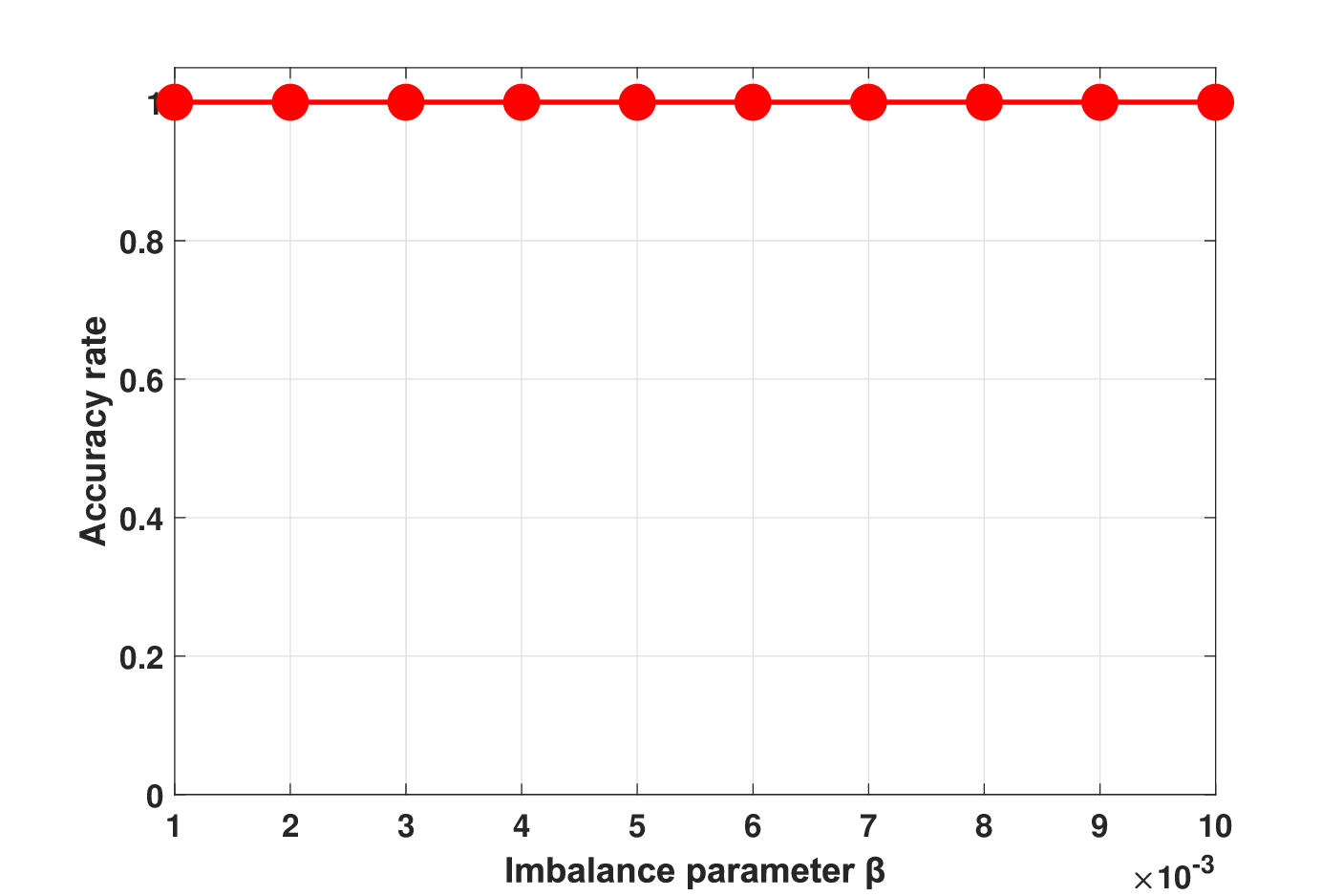}}
\caption{Numerical results of Experiment 3.}
\label{fig:ex3} 
\end{figure}
\subsubsection{Effect of the constant factor in the separation condition}
\label{sec:exp4}

In this experiment, we investigate the behaviour of the CSVT estimator when the minimal centre distance \(\Delta\) is smaller than the theoretical lower bound used in the previous experiments. Instead of fixing the constant to \(2\sqrt{2}\) as in the previous experiments, we introduce a scaling factor \(\gamma > 0\) and set the minimal centre distance to
\[
\Delta = \gamma\cdot2\sqrt{2} \frac{1}{\sqrt{\beta}} \sqrt{\frac{K}{n}} \bigl(\sqrt{p}+\sqrt{n}+\log n\bigr).
\]

We then vary \(\gamma\) over a grid of values, e.g. \(\gamma \in \{0.1, 0.2, \ldots,1\}\), and report the accuracy rate (proportion of runs where \(\widehat{K}=K\)) for each \(\gamma\). Two distinct scenarios are considered as before.

\begin{itemize}
    \item \textbf{Sample‑rich  regime:}   \(n = 10^5\), \(p = 200\), \(K = 200\), \(\beta = 0.01\). 
    \item \textbf{High‑dimensional  regime:} \(n = 200\), \(p = 10^5\), \(K = 10\), \(\beta = 0.5\). 
\end{itemize}

This design allows us to assess the robustness of the CSVT estimator when the theoretical separation condition is violated, and to identify the range of \(\gamma\) for which the estimator remains reliable. Figure~\ref{fig:ex4} shows the accuracy of the CSVT estimator as the separation distance between cluster centers decreases. In both the sample‑rich and high‑dimensional regimes, the estimator perfectly recovers the true number of components for all scaling factors \(\gamma \ge 0.2\). For \(\gamma = 0.1\), which means the actual center distance is one tenth of the theoretical lower bound required by Theorem~\ref{thm:main}, the estimator does not recover \(K\) correctly. These results indicate that the CSVT estimator is robust in practice. It works reliably when the cluster centers are not extremely close, and the theoretical condition provides a safe guideline without being unnecessarily restrictive.
\begin{figure}[!htbp]
\centering
\resizebox{\columnwidth}{!}{
{\includegraphics[width=3\textwidth]{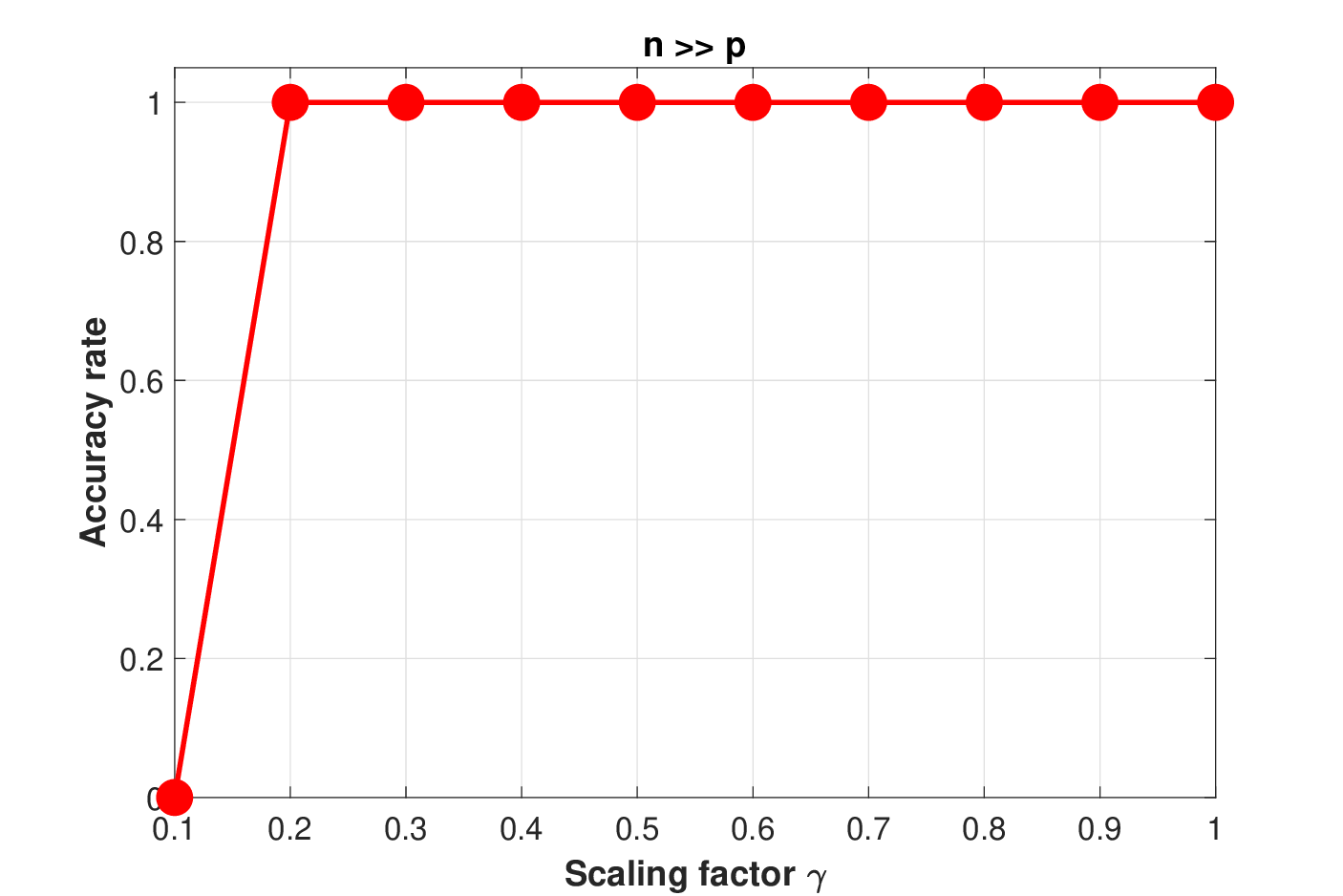}}
{\includegraphics[width=3\textwidth]{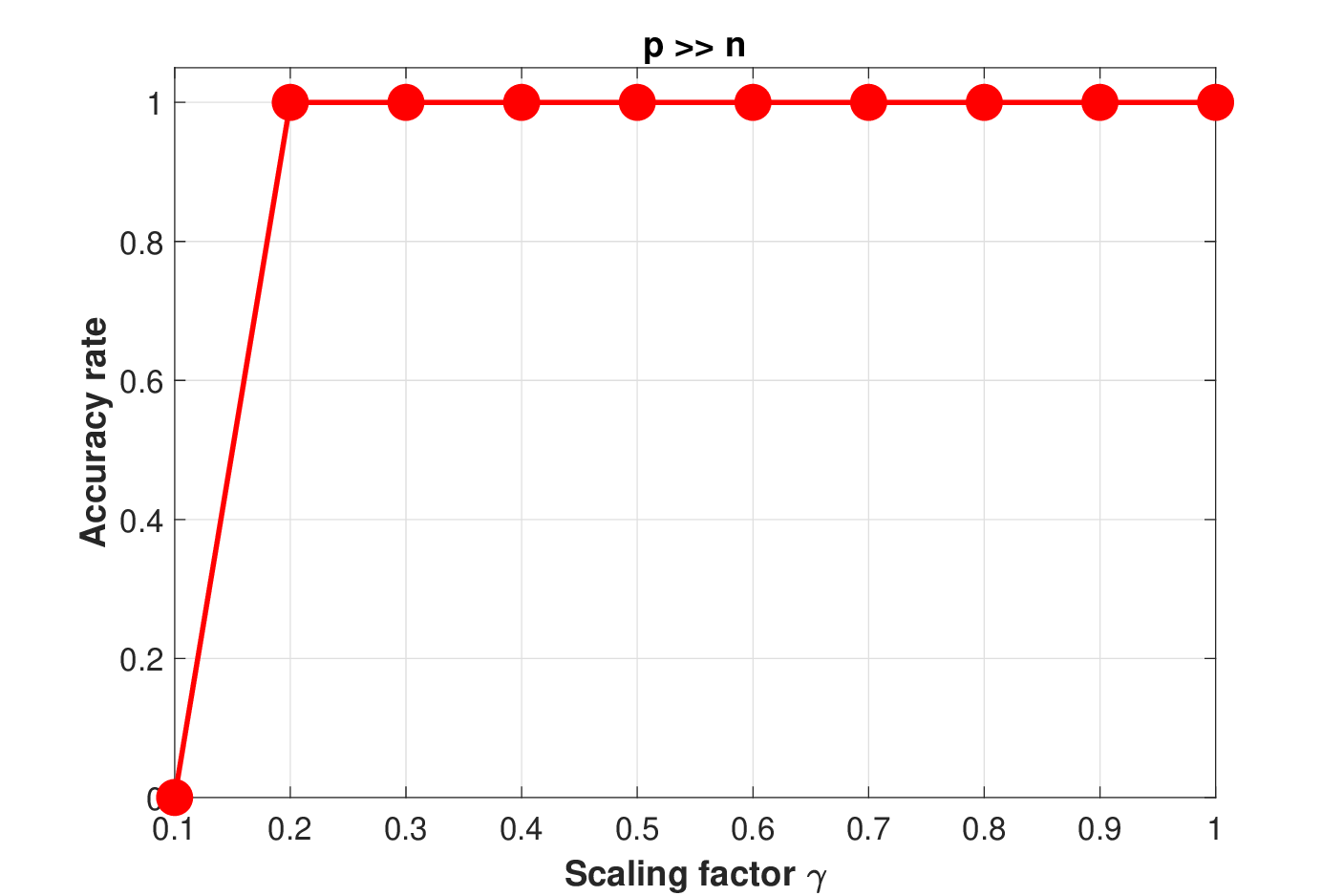}}
}
\caption{Numerical results of Experiment 4.}
\label{fig:ex4} 
\end{figure}
\subsubsection{Performance under heterogeneous noise}
\label{sec:exp5}

All previous simulations assumed the noise matrix $\mathbf{E}$ has i.i.d. $\mathcal{N}(0,1)$ entries, which satisfies the theoretical condition of Theorem~\ref{thm:main}. In practice, the noise may be heteroscedastic: different dimensions can have different variances. To assess the robustness of CSVT beyond its theoretical assumptions, we now consider a more challenging noise structure where the noise condition is violated. In detail, for each independent replication, we generate the noise matrix $\mathbf{E}\in\mathbb{R}^{p\times n}$ with entries  
\[
\mathbf{E}_{ji} \stackrel{\text{i.i.d.}}{\sim} \mathcal{N}\bigl(0,\;\eta_i^2\bigr), \qquad i\in[n],j\in[p]
\]  
where the per‑coordinate variance $\eta_i^2$ is drawn independently from a uniform distribution:  
\[
\eta_i \sim \mathrm{Uniform}\bigl(0.5,\; \eta_{\max}\bigr),\qquad i\in[n].
\]  

The parameter $\eta_{\max}$ controls the degree of heteroscedasticity and we let $\eta_{\max}$ vary from $0.1$ to $1.5$ in steps of $0.1$. For each value of $\eta_{\max}$, the variances $\eta_i^2$ are resampled in every repetition. All other data generation settings (centres, cluster sizes, etc.) remain unchanged from the previous experiments. We design two distinct scenarios that represent the most challenging settings examined earlier:
\begin{itemize}
    \item \textbf{Sample‑rich  regime:}  
    $n = 10^6$, $p = 200$, $K = 200$, $\beta = 0.01$.  

    \item \textbf{High‑dimensional regime:}  
    $n = 200$, $p = 10^6$, $K = 10$, $\beta = 0.5$.
\end{itemize}

For each scenario and each value of $\eta_{\max}$, we generate $100$ independent data sets. The CSVT estimator is applied exactly as described in Algorithm~\ref{alg:CSVT} using the \emph{same} threshold $T = \sqrt{p}+\sqrt{n}+\log n$ (i.e., still assuming unit noise variance). No whitening or pre‑estimation of the noise variances is performed. We record the accuracy rate for each configuration. 

Figure~\ref{fig:ex5} shows how the CSVT estimator performs when the noise variance varies across dimensions. In the sample‑rich regime, the estimator achieves perfect accuracy for all $\eta_{\max} \le 1.4$. However, once $\eta_{\max}$ exceeds $1.4$, the accuracy drops sharply to zero. This indicates that the method can tolerate mild heteroscedasticity but fails when the noise becomes too large. In the high‑dimensional regime (where the dimension is much larger than the sample size), the estimator remains perfectly correct as long as the largest noise variance is no more than $1$. For $\eta_{\max} > 1$, the accuracy immediately falls to zero.

These results demonstrate that the CSVT estimator handles small-to-moderate differences in noise variance well, despite the theoretical assumption of homoscedasticity. However, there exists a clear threshold beyond which the estimator breaks down, especially in high dimensions where the noise level is amplified. The findings highlight the method's practical robustness within a reasonable range of noise heterogeneity, while also revealing its limits under extreme heteroscedasticity.
\begin{figure}[!htbp]
\centering
\resizebox{\columnwidth}{!}{
{\includegraphics[width=3\textwidth]{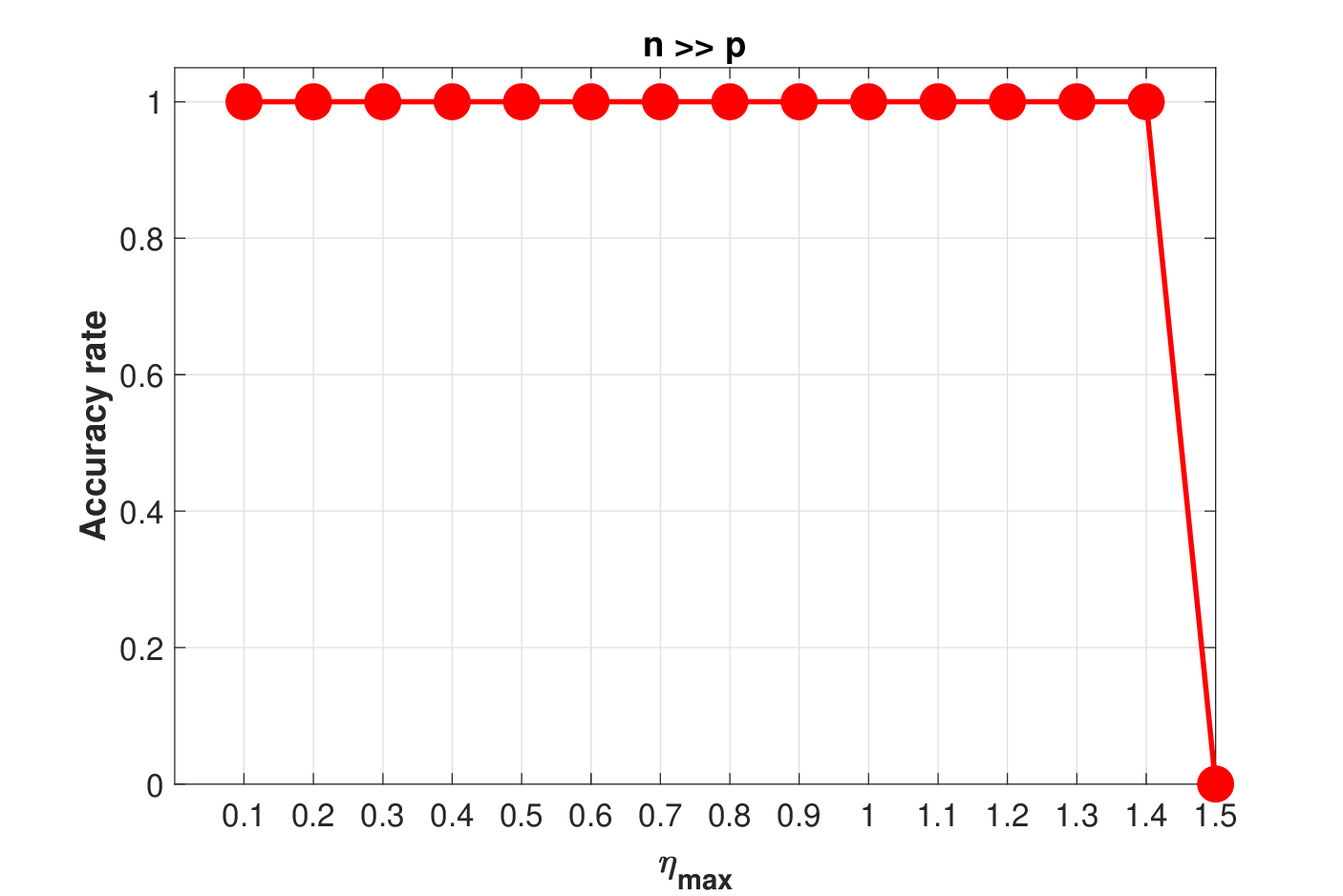}}
{\includegraphics[width=3\textwidth]{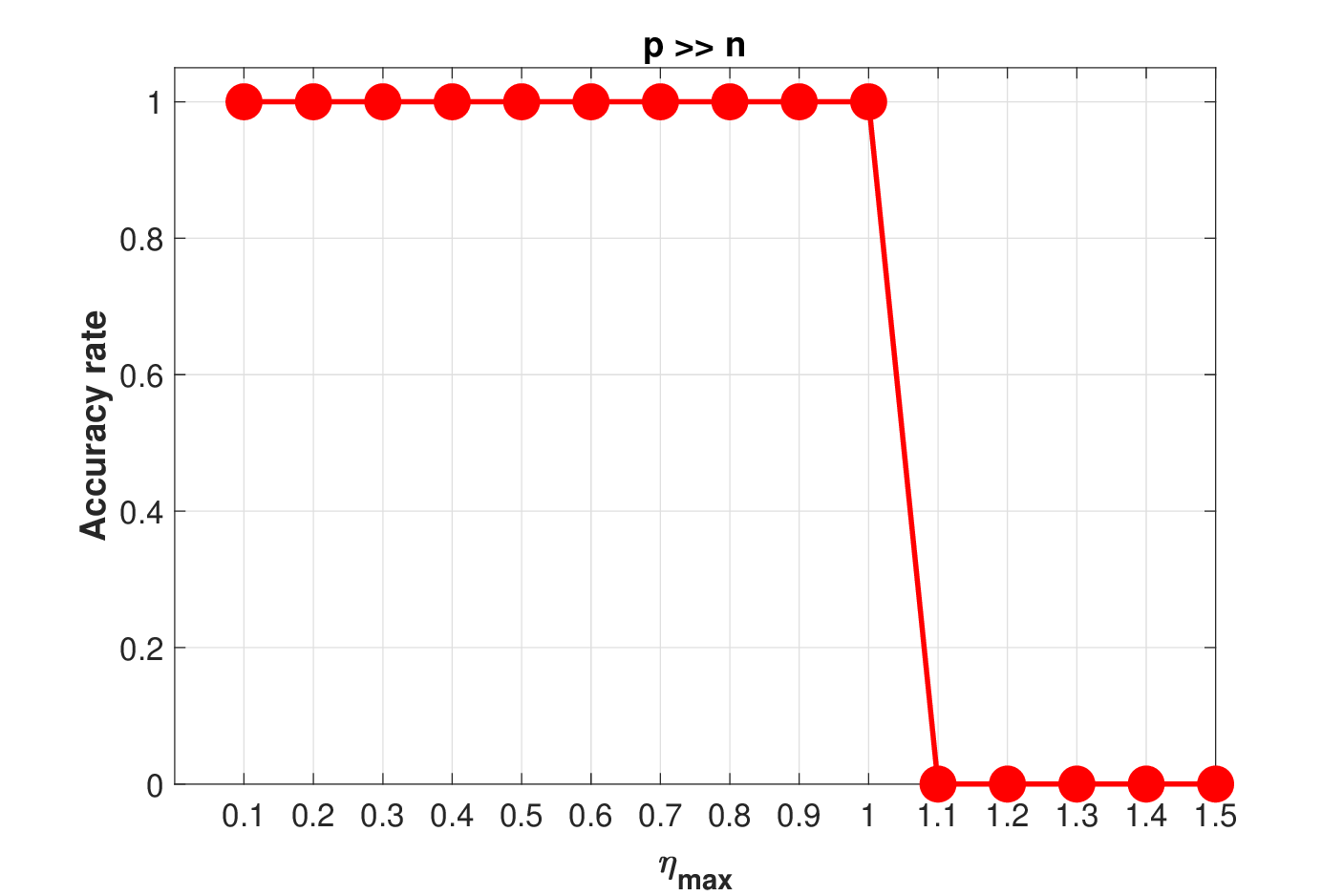}}
}
\caption{Numerical results of Experiment 5.}
\label{fig:ex5} 
\end{figure}
\section{Real data analysis}
\label{sec:realdata}

This section evaluates the practical performance of the CSVT estimator on four real-world datasets that are common benchmarks in clustering literature. For each data, the true number of underlying groups \(K\) is known from domain knowledge. Our goal is to assess whether the proposed CSVT estimator can recover this true \(K\) from the data alone. The four datasets are briefly described as follows.

\begin{itemize}
    \item Iris dataset \citep{fisher1936use}: 
    This dataset contains \(n = 150\) samples from three species of iris flowers: setosa, versicolor, and virginica. Each sample is characterized by \(p = 4\) continuous measurements: sepal length, sepal width, petal length, and petal width. The true number of clusters is \(K = 3\), corresponding to the three species.

    \item Crab dataset \citep{campbell1974multivariate}: 
    This dataset consists of \(n = 200\) rock crab specimens of the genus \textit{Leptograpsus}. Each specimen is described by \(p = 5\) morphological measurements. In the original data, the crabs have one of two color forms and differ in sex, leading to four pre-defined groups. However, in this study, we follow the common practice of ignoring the color form and focus solely on sex differentiation. Consequently, the true number of clusters is set to \(K = 2\), corresponding to male and female crabs. 

    \item USPS handwritten digit database \citep{hull2002database}: 
    This dataset contains \(n = 7291\) grayscale images of handwritten digits from 0 to 9. Each image is of size \(16 \times 16\) pixels, which is vectorized to a \(p = 256\)-dimensional vector of pixel intensities. The true number of clusters is \(K = 10\), corresponding to the ten digit classes.
\item Poker dataset \citep{poker_hand_158}: This dataset is sourced from the UCI Poker Hand dataset . It contains \(n = 25010\) samples. Each sample describes a hand of five playing cards drawn from a standard deck of 52. Each card is represented by two attributes (suit and rank), resulting in \(p = 10\) predictive attributes. The order of cards in the hand is important. The true number of clusters is \(K = 10\), corresponding to the ten distinct poker hand types.
\end{itemize}

For each dataset, we apply the CSVT estimator and report the estimated number of components \(\widehat{K}\) and the running time. Table~\ref{tab:realdata_results} summarizes the characteristics of each dataset along with the estimated \(\widehat{K}\) and the computation time. The results are discussed below.

\begin{table}[htbp]
\centering
\caption{Dataset characteristics, CSVT estimation results, and average running time over 1000 runs for each data.}
\label{tab:realdata_results}
\begin{tabular}{lcccc c}
\toprule
Dataset & \(n\) & \(p\) & True \(K\) & Estimated \(\widehat{K}\) & Time (s) \\
\midrule
Iris     & 150   & 4   & 3 & 2 & 0.000026 \\
Crab     & 200   & 5   & 2 & 2 & 0.000037 \\
USPS     & 7291  & 256 & 10 & 10 & 0.062435 \\
Poker    & 25010 & 10  & 10 & 10 & 0.005061 \\
\bottomrule
\end{tabular}
\end{table}

\begin{itemize}
\item For the Iris data, CSVT estimates \(\widehat{K} = 2\). This result does not contradict the data but instead reflects the actual geometry of the feature space. The setosa class is well separated from the other two, while the versicolor and virginica classes overlap substantially. Under a spectral thresholding criterion that relies on centered singular values, the data support only two effective clusters. One cluster captures the isolated setosa group, and the other represents the combined mass of the two overlapping classes. Hence \(\widehat{K}=2\) is a faithful reflection of the data geometry rather than an estimation error. CSVT's average running time over 1000 runs for Iris is 0.000026 seconds.
\item For the Crab data, CSVT estimates \(\widehat{K} = 2\). The two sexes are well separated in the measurement space, and the CVST estimator recovers the true number without difficulty. CSVT's average running time over 1000 runs for Crab is 0.000037 seconds.
\item For the USPS data, CSVT estimates \(\widehat{K} = 10\). This outcome is notable because the dataset has a large sample size (\(n = 7291\)) relative to the feature dimension (\(p = 256\)), and the pixel intensities do not satisfy any simple distributional assumption. Nevertheless, the ten digit classes are sufficiently distinct in the centered singular value spectrum to be counted correctly. The estimator handles the large number of clusters and the large sample size with ease. CSVT's average running time over 1000 runs  for this data is 0.062435 seconds.
\item For the Poker data, CSVT estimates \(\widehat{K} = 10\), which matches the true number of poker hand types. Notably, for this dataset, the feature dimension equals the number of clusters (\(p = K = 10\)). The estimator successfully identifies the ten distinct hand classes from the 10-dimensional feature space. The average running time of CSVT over 1000 runs is 0.005061 seconds for this data.
\end{itemize}

These results demonstrate that the proposed CSVT estimator performs well on real-world data across a variety of sample sizes, feature dimensions, and numbers of clusters. In all four datasets, the estimated \(\widehat{K}\) either matches the true number of clusters or provides a meaningful reflection of the underlying data geometry. Moreover, the computational cost is very low: even for the largest dataset (USPS, \(n=7291, p=256\)), the average running time over 1000 runs is only about 0.06 seconds.
\section{Conclusion}\label{sec:conclusion}
This paper provides a simple and practical tool for estimating the number of components in a Gaussian mixture model. The method has a rigorous theoretical guarantee. It works under remarkably mild conditions and is valid for any dimension and sample size, including the high dimensional setting where the dimension far exceeds the sample size. The true number of clusters can be large and is allowed to grow with the smaller of the dimension and the sample size. The required separation between the cluster centers is very weak. The algorithm is extremely fast, requiring only a single centering operation and one singular value decomposition. The proof is self contained and uses only elementary matrix inequalities and standard concentration bounds. This work contributes to the field of high dimensional unsupervised learning by offering a computationally efficient and theoretically sound solution to a fundamental model selection problem.

Several directions remain for future work. A fundamental open problem is whether the separation condition in Theorem \ref{thm:main} is also necessary: does there exist a constant \(c>0\) such that whenever \(\Delta \le c\frac{\kappa}{\sqrt{\beta}}\sqrt{\frac{K}{n}}(\sqrt{p}+\sqrt{n})\), no estimator can consistently recover the true number of components \(K\)? Resolving this minimax lower bound would establish the optimality of CSVT, but it remains a challenging open problem due to the difficulty of constructing statistically indistinguishable mixtures with \(K\) and \(K+1\) components under small separation. One natural extension is to relax the assumption that the covariance matrix is the identity. If the covariance is known to be isotropic but unknown, the noise level must be estimated from the data, but establishing theoretical guarantees would require new concentration arguments. Another direction is to consider a common but unknown full covariance matrix. One could apply a whitening step based on a preliminary covariance estimate before centering and thresholding. The key challenge is to control how the estimation error in the whitening transformation affects the rank of the signal. A fourth avenue is to move beyond Gaussian noise. The current analysis relies on concentration bounds for Gaussian random matrices. It would be valuable to examine whether the same centering and thresholding idea remains effective for other noise distributions with subgaussian tails, where similar concentration inequalities are available. Finally, beyond the current model, if consistent estimation of \(K\) is established for more complex GMMs (e.g., unknown or non‑spherical covariances), the resulting separation condition will inevitably involve the same key model parameter \(\Delta\). Studying the minimax lower bound for \(\Delta\) in such extended settings is therefore equally challenging and promising.

\bmhead{Acknowledgements} Qing's work was sponsored by the Scientific Research Foundation of Chongqing University of Technology (Grant No. 2024ZDR003), and the Science and Technology Research Program of Chongqing Municipal Education Commission (Grant No. KJQN202401168).
\bmhead{Author Contributions} Huan Qing: Conceptualization, Methodology, Investigation, Software, Formal analysis, Data curation, Writing-original draft, Writing-reviewing \& editing, Funding acquisition.
\section*{Declarations}
\textbf{Conflict of interest} The author declares no conflict of interest.
\begin{appendices}
\section{Technical proofs}\label{app:proofs}
\subsection{Proof of Lemma \ref{lem:signal}}
\begin{proof}
\textbf{Step 1. Decomposition.}
Write $\mathbf{P}=\mathbf{M}\mathbf{Z}^\top$ with $\mathbf{M}\in\mathbb{R}^{p\times K}$ (full column rank) and $\mathbf{Z}\in\{0,1\}^{n\times K}$ the membership matrix. Then
\begin{align*}
\widetilde{\mathbf{P}} = \mathbf{P}\mathbf{H} = \mathbf{M}\mathbf{Z}^\top\mathbf{H} \equiv \mathbf{M}\mathbf{A},
\qquad \mathbf{A}:=\mathbf{Z}^\top\mathbf{H}\in\mathbb{R}^{K\times n}.
\end{align*}

\textbf{Step 2. Lower bound via $\sigma_K(\mathbf{M})$ and $\sigma_{K-1}(\mathbf{A})$.}
Since $\mathbf{M}$ has rank $K$ and $\operatorname{rank}(\mathbf{A})=K-1$ (affine independence plus centering), for any $i\le\operatorname{rank}(\mathbf{A})$, we have
\begin{align*}
\sigma_i(\mathbf{M}\mathbf{A}) \ge \sigma_{\min}(\mathbf{M})\,\sigma_i(\mathbf{A}) = \sigma_K(\mathbf{M})\,\sigma_i(\mathbf{A}).
\end{align*}

Taking $i=K-1$ gives
\begin{align}
\sigma_{K-1}(\widetilde{\mathbf{P}}) \ge \sigma_K(\mathbf{M})\;\sigma_{K-1}(\mathbf{A}). \label{eq:step2}
\end{align}

\textbf{Step 3. Compute $\mathbf{A}\mathbf{A}^\top$ and the correct orthogonality constraint.}
\begin{align*}
\mathbf{A}\mathbf{A}^\top = \mathbf{Z}^\top\mathbf{H}\mathbf{Z}
= \mathbf{Z}^\top(\mathbf{I}_n-\tfrac{1}{n}\mathbf{1}_n\mathbf{1}_n^\top)\mathbf{Z}
= \mathbf{D} - \tfrac{1}{n}\mathbf{n}\mathbf{n}^\top,
\end{align*}
where $\mathbf{D}=\operatorname{diag}(n_1,\dots,n_K)$ and $\mathbf{n}=(n_1,\dots,n_K)^\top$.
One checks $\mathbf{A}\mathbf{A}^\top\mathbf{1}_K=\mathbf{0}$, so $\mathbf{1}_K$ is the zero eigenvector and $\operatorname{rank}(\mathbf{A}\mathbf{A}^\top)=K-1$.
By the Courant–Fischer theorem, we have
\begin{align}
\sigma_{K-1}^2(\mathbf{A}) = \min_{\substack{\mathbf{x}\in\mathbb{R}^K\\ \mathbf{x}^\top\mathbf{1}_K=0,\;\|\mathbf{x}\|=1}}
\mathbf{x}^\top\!\left(\mathbf{D}-\tfrac{1}{n}\mathbf{n}\mathbf{n}^\top\right)\!\mathbf{x}. \label{eq:step3}
\end{align}

\textbf{Step 4. Prove $\sigma_{K-1}(\mathbf{A}) \ge \sqrt{\min_k n_k}$.}
Let $m=\min_k n_k$ and write $n_k=m+d_k$ with $d_k\ge0$, so $\sum_k d_k = n-Km$.
For any $\mathbf{x}$ with $\mathbf{x}^\top\mathbf{1}_K=0$, $\|\mathbf{x}\|=1$, we have
\begin{align*}
\mathbf{x}^\top\!\left(\mathbf{D}-\tfrac{1}{n}\mathbf{n}\mathbf{n}^\top\right)\!\mathbf{x}
&= \sum_{k=1}^K n_k x_k^2 - \frac{1}{n}\Bigl(\sum_{k=1}^K n_k x_k\Bigr)^2.
\end{align*}

Because $\sum_k x_k=0$, we have
\begin{align*}
\sum_{k=1}^K n_k x_k = \sum_{k=1}^K (m+d_k)x_k = \sum_k d_k x_k.
\end{align*}

Hence, by Cauchy–Schwarz inequality, we have
\begin{align}
\Bigl(\sum_k d_k x_k\Bigr)^2 \le \Bigl(\sum_k d_k\Bigr)\Bigl(\sum_k d_k x_k^2\Bigr)
= (n-Km)\sum_k d_k x_k^2. \label{eq:cs}
\end{align}

We also have
\begin{align*}
\sum_k n_k x_k^2 = \sum_k (m+d_k)x_k^2 = m + \sum_k d_k x_k^2.
\end{align*}

Therefore, we get
\begin{align*}
\mathbf{x}^\top\!\left(\mathbf{D}-\tfrac{1}{n}\mathbf{n}\mathbf{n}^\top\right)\!\mathbf{x}
&= m + \sum_k d_k x_k^2 - \frac{1}{n}\Bigl(\sum_k d_k x_k\Bigr)^2 \\
&\ge m + \sum_k d_k x_k^2 - \frac{1}{n}(n-Km)\sum_k d_k x_k^2 \quad\text{(by Equation (\ref{eq:cs}))}\\
&\ge m + \sum_k d_k x_k^2 - \sum_k d_k x_k^2 = m,
\end{align*}
where we used $\frac{1}{n}(n-Km)\le 1$. Since this holds for every admissible $\mathbf{x}$, Equation (\ref{eq:step3}) yields
\begin{align}
\sigma_{K-1}^2(\mathbf{A}) \ge m \quad\Longrightarrow\quad
\sigma_{K-1}(\mathbf{A}) \ge \sqrt{m} = \sqrt{\min_k n_k}. \label{eq:step4}
\end{align}

\textbf{Step 5. Lower bound $\sigma_K(\mathbf{M}) \ge \Delta/(2\kappa)$.}
From the definition $\kappa = \sigma_1(\mathbf{M})/\sigma_K(\mathbf{M})$, we have $\sigma_K(\mathbf{M}) = \sigma_1(\mathbf{M})/\kappa$.
For any two distinct centres, we have
\begin{align*}
\Delta \le \|\boldsymbol{\mu}_i-\boldsymbol{\mu}_j\|
= \|\mathbf{M}(\mathbf{e}_i-\mathbf{e}_j)\|
\le \|\mathbf{M}\|\,\|\mathbf{e}_i-\mathbf{e}_j\|
= \sigma_1(\mathbf{M})\sqrt{2}.
\end{align*}

Thus, we get $\sigma_1(\mathbf{M}) \ge \Delta/\sqrt{2}$, and consequently we obtain
\begin{align}
\sigma_K(\mathbf{M}) \ge \frac{\Delta/\sqrt{2}}{\kappa}. \label{eq:step5}
\end{align}

\textbf{Step 6. Combine.}
Inserting Equations (\ref{eq:step4}) and (\ref{eq:step5}) into Equation (\ref{eq:step2}) gives
\begin{align*}
\sigma_{K-1}(\widetilde{\mathbf{P}})
\ge \frac{\Delta}{\sqrt{2}\kappa}\,\sqrt{\min_k n_k}
= \frac{\Delta}{\kappa}\,\sqrt{\frac{\beta n}{2K}}.
\end{align*}

This completes the proof.
\end{proof}
\subsection{Proof of Theorem \ref{thm:main}}
\begin{proof}
We treat the two cases separately.

\noindent\textbf{Case $K=1$.} 
Then $\widetilde{\mathbf{P}}=\mathbf{0}$, so $\widetilde{\mathbf{X}}=\widetilde{\mathbf{E}}$. We need the following technical result to bound \(\|\widetilde{\mathbf{E}}\|\).
\begin{lem}\label{lem:noise}
Let $\mathbf{E}\in\mathbb{R}^{p\times n}$ have i.i.d. $\mathcal{N}(0,1)$ entries. For any $t\ge 0$, we have
\begin{align*}
    \mathbb{P}\bigl(\|\mathbf{E}\| \ge \sqrt{p}+\sqrt{n}+t\bigr) \le e^{-t^2/2}.
\end{align*}
\end{lem}

\begin{proof}
The inequality follows directly from Lemma B.1 of \citep{loffler2021optimality}.
\end{proof}

By Lemma~\ref{lem:noise}, for the chosen $t_n$, we have
\[
\mathbb{P}\bigl(\|\mathbf{E}\| \ge d+t_n\bigr) \le e^{-t_n^2/2} \to 0.
\]

Define the event $\mathcal{A} = \{\|\mathbf{E}\| \le d+t_n\}$. Then $\mathbb{P}(\mathcal{A})\to 1$ and on $\mathcal{A}$, we have
\[
\|\widetilde{\mathbf{E}}\| \le \|\mathbf{E}\| \le d+t_n = T.
\]

Hence, all singular values of $\widetilde{\mathbf{X}}$ satisfy $\widehat{\sigma}_i \le T$ on $\mathcal{A}$. 
Because $\widetilde{\mathbf{X}}=\widetilde{\mathbf{E}}$ and $\mathbf{E}$ has an absolutely continuous distribution, $\mathbb{P}(\widehat{\sigma}_i = T)=0$ for each $i$. Since there are only $\min(p,n)$ singular values, the union bound gives
\[
\mathbb{P}(\exists i: \widehat{\sigma}_i = T) \le \sum_{i=1}^{\min(p,n)} \mathbb{P}(\widehat{\sigma}_i = T) = 0.
\]

Therefore, we get
\[
\mathbb{P}(\widehat{\sigma}_i < T\ \forall i) \ge \mathbb{P}(\mathcal{A}) - \mathbb{P}(\exists i: \widehat{\sigma}_i = T) = \mathbb{P}(\mathcal{A}) \to 1,
\]
which gives $r=0$ and the algorithm outputs $\widehat{K}=1$ with probability tending to one.

\noindent\textbf{Case $K\ge2$.} 
From Lemma~\ref{lem:signal}, we obtain
\[
\sigma_{K-1}(\widetilde{\mathbf{P}}) \ge \frac{\Delta}{\kappa}\sqrt{\frac{\beta n}{2K}}.
\]

Insert the lower bound on $\Delta$ from the separation condition, we get
\[
\sigma_{K-1}(\widetilde{\mathbf{P}}) \ge \frac{1}{\kappa}\sqrt{\frac{\beta n}{2K}}\cdot
2\sqrt{2}\,\frac{\kappa}{\sqrt{\beta}}\sqrt{\frac{K}{n}}\,T
= 2T.
\]

Now restrict to the same event $\mathcal{A}$ (which holds with probability $\to 1$). On $\mathcal{A}$, we have $\|\widetilde{\mathbf{E}}\| \le T$. Applying Weyl's inequality to $\widetilde{\mathbf{X}} = \widetilde{\mathbf{P}}+\widetilde{\mathbf{E}}$ gives
\[
\widehat{\sigma}_{K-1} \ge \sigma_{K-1}(\widetilde{\mathbf{P}}) - \|\widetilde{\mathbf{E}}\|
\ge 2T - T = T,
\]
\[
\widehat{\sigma}_{K} \le \sigma_{K}(\widetilde{\mathbf{P}}) + \|\widetilde{\mathbf{E}}\|
= 0 + \|\widetilde{\mathbf{E}}\| \le T.
\]

Because the singular values of $\widetilde{\mathbf{X}}$ are continuous functions of the data matrix, and the data have an absolutely continuous distribution, we have $\mathbb{P}(\widehat{\sigma}_{K-1}=T \text{ or } \widehat{\sigma}_K=T)=0$. Hence, we have
\[
\mathbb{P}(\widehat{\sigma}_{K-1}>T,\ \widehat{\sigma}_K<T) \ge \mathbb{P}(\mathcal{A}) - \mathbb{P}(\widehat{\sigma}_{K-1}=T \text{ or } \widehat{\sigma}_K=T) \to 1.
\]

Thus, with probability tending to one, $\widehat{\sigma}_{K-1}>T$ and $\widehat{\sigma}_K<T$. Since the singular values are non‑increasing, it follows that \(\widehat{\sigma}_i > T\) for all \(i = 1,\dots,K-1\), while \(\widehat{\sigma}_K < T\). Hence exactly \(K-1\) singular values exceed \(T\), i.e. \(r = K-1\).  The algorithm then returns \(\widehat{K} = r+1 = K\).  Thus \(\mathbb{P}(\widehat{K}=K) \to 1\), completing the proof.
\end{proof}

\subsection{Proof of Lemma \ref{pathology}}
\begin{proof}
Since $\mathbf{c}=t\mathbf{v}$, we have
\[
\mathbf{P} = t\mathbf{v}\mathbf{1}_n^\top + \mathbf{v}\mathbf{w}^\top = \mathbf{v}\bigl(t\mathbf{1}_n^\top + \mathbf{w}^\top\bigr),
\]
which is an outer product of $\mathbf{v}$ and the row vector $t\mathbf{1}_n^\top+\mathbf{w}^\top$. For any $t>1$, each entry of this row vector is $t\pm 1 > 0$, hence non-zero. Therefore, we have $\operatorname{rank}(\mathbf{P}) = 1$ and $\sigma_2(\mathbf{P}) = 0$. Then, by Weyl's inequality, we have
\[
\sigma_2(\mathbf{X}) \le \sigma_2(\mathbf{P}) + \|\mathbf{E}\| = \|\mathbf{E}\|.
\]

By Lemma \ref{lem:noise}, we know that $\|\mathbf{E}\| \le T$ with probability tending to $1$, so $\sigma_2(\mathbf{X}) < T$ w.h.p. For the largest singular value, we have
\[
\sigma_1(\mathbf{X}) \ge \sigma_1(\mathbf{P}) - \|\mathbf{E}\|.
\]

Now $\sigma_1(\mathbf{P}) = \|\mathbf{v}\|\cdot\|t\mathbf{1}_n^\top+\mathbf{w}^\top\| \ge \frac{\Delta}{2}\bigl(t\sqrt{n} - \sqrt{n}\bigr)$ by the reverse triangle inequality. Choosing $t$ large enough ensures $\sigma_1(\mathbf{P}) > 2T$, hence we have $\sigma_1(\mathbf{X}) > T$ w.h.p. Consequently, only exactly one singular value exceeds $T$, yielding $\widehat{K}_{\text{raw}} = 1$ while the true number of components is $2$. 
\end{proof}
\subsection{Additional lemma}
\begin{lem}\label{lem:uncentered_signal}
Assume that $\mathbf{M}$ has full column rank (so $\sigma_K(\mathbf{M})>0$).  
When $K\ge 2$, the $K$-th singular value of the uncentered signal matrix $\mathbf{P}=\mathbf{M}\mathbf{Z}^{\top}$ satisfies
\[
\sigma_K(\mathbf{P}) \;\ge\; \frac{\Delta}{\kappa}\,\sqrt{\frac{\beta n}{2K}},
\]
where $\beta = \frac{\min_k n_k}{n/K}\in(0,1]$, $\Delta = \min_{k\neq\ell}\|\boldsymbol{\mu}_k-\boldsymbol{\mu}_\ell\|$, and $\kappa = \frac{\sigma_1(\mathbf{M})}{\sigma_K(\mathbf{M})}$.
\end{lem}

\begin{proof}
Let $\mathbf{D}=\operatorname{diag}(n_1,\dots,n_K)$.  Because $\mathbf{Z}^{\top}\mathbf{Z}=\mathbf{D}$, we can write $\mathbf{Z}^{\top} = \mathbf{D}^{1/2}\mathbf{R}$ where $\mathbf{R}\in\mathbb{R}^{K\times n}$ satisfies $\mathbf{R}\mathbf{R}^{\top}=\mathbf{I}_K$ (the rows of $\mathbf{R}$ are orthonormal).  Consequently,
\[
\mathbf{P} = \mathbf{M}\mathbf{D}^{1/2}\mathbf{R}.
\]

Since $\mathbf{R}$ has orthonormal rows, there exists an orthogonal matrix $\mathbf{Q}\in\mathbb{R}^{n\times n}$ such that
\[
\mathbf{R} = [\mathbf{I}_K,\;\mathbf{0}]\,\mathbf{Q}.
\]

Thus, we have
\[
\mathbf{P} = \mathbf{M}\mathbf{D}^{1/2}[\mathbf{I}_K,\;\mathbf{0}]\,\mathbf{Q}
= [\mathbf{M}\mathbf{D}^{1/2},\;\mathbf{0}]\,\mathbf{Q}.
\]

Multiplying by an orthogonal matrix on the right does not change the singular values, so $\sigma_K(\mathbf{P}) = \sigma_K(\mathbf{M}\mathbf{D}^{1/2})$. Now $\mathbf{M}\mathbf{D}^{1/2}$ is a $p\times K$ matrix.  For any $\mathbf{x}\in\mathbb{R}^K$ with $\|\mathbf{x}\|=1$, we have
\[
\|\mathbf{M}\mathbf{D}^{1/2}\mathbf{x}\| \ge \sigma_K(\mathbf{M})\,\|\mathbf{D}^{1/2}\mathbf{x}\|
\ge \sigma_K(\mathbf{M})\,\sigma_K(\mathbf{D}^{1/2})\,\|\mathbf{x}\|
= \sigma_K(\mathbf{M})\,\sqrt{\min_k n_k}.
\]

Taking the minimum over all unit vectors gives $\sigma_K(\mathbf{M}\mathbf{D}^{1/2}) \ge \sigma_K(\mathbf{M})\,\sqrt{\min_k n_k}$. Therefore, we get
\[
\sigma_K(\mathbf{P}) \ge \sigma_K(\mathbf{M})\,\sqrt{\frac{\beta n}{K}}.
\]

Finally, from the definition of $\Delta$, we get
\[
\Delta \le \|\boldsymbol{\mu}_i-\boldsymbol{\mu}_j\| \le \|\mathbf{M}\|\,\|\mathbf{e}_i-\mathbf{e}_j\|
= \sigma_1(\mathbf{M})\sqrt{2},
\]
so $\sigma_1(\mathbf{M}) \ge \Delta/\sqrt{2}$.  Using $\kappa = \sigma_1(\mathbf{M})/\sigma_K(\mathbf{M})$, we obtain $\sigma_K(\mathbf{M}) \ge \Delta/(\sqrt{2}\kappa)$.  Thus, we obtain
\[
\sigma_K(\mathbf{P}) \ge \frac{\Delta}{\sqrt{2}\kappa}\,\sqrt{\frac{\beta n}{K}}
= \frac{\Delta}{\kappa}\,\sqrt{\frac{\beta n}{2K}},
\]
which completes the proof.
\end{proof}
\end{appendices}


\bibliography{reference}

@article{jain1999data,
  title={Data clustering: a review},
  author={Jain, Anil K and Murty, M Narasimha and Flynn, Patrick J},
  journal={ACM Computing Surveys (CSUR)},
  volume={31},
  number={3},
  pages={264--323},
  year={1999},
  publisher={Acm New York, NY, USA}
}

@article{pearson1894iii,
  title={III. Contributions to the mathematical theory of evolution},
  author={Pearson, Karl},
  journal={Proceedings of the Royal Society of London},
  volume={54},
  number={326-330},
  pages={329--333},
  year={1894},
  publisher={The Royal Society London}
}

@article{jain2010data,
  title={Data clustering: 50 years beyond K-means},
  author={Jain, Anil K},
  journal={Pattern Recognition Letters},
  volume={31},
  number={8},
  pages={651--666},
  year={2010},
  publisher={Elsevier}
}

@article{ndaoud2022sharp,
  title={Sharp optimal recovery in the two component Gaussian mixture model},
  author={Ndaoud, Mohamed},
  journal={Annals of Statistics},
  volume={50},
  number={4},
  pages={2096--2126},
  year={2022},
  publisher={Institute of Mathematical Statistics}
}

@article{li2025exact,
  title={Exact recovery of community detection in k-community Gaussian mixture models},
  author={Li, Zhongyang},
  journal={European Journal of Applied Mathematics},
  volume={36},
  number={3},
  pages={491--523},
  year={2025},
  publisher={Cambridge University Press}
}

@article{chen2024achieving,
  title={Achieving optimal clustering in Gaussian mixture models with anisotropic covariance structures},
  author={Chen, Xin and Zhang, Anderson Ye},
  journal={Advances in Neural Information Processing Systems},
  volume={37},
  pages={113698--113741},
  year={2024}
}

@article{lu2016statistical,
  title={Statistical and computational guarantees of lloyd's algorithm and its variants},
  author={Lu, Yu and Zhou, Harrison H},
  journal={arXiv preprint arXiv:1612.02099},
  year={2016}
}

@article{chen2021cutoff,
  title={Cutoff for exact recovery of gaussian mixture models},
  author={Chen, Xiaohui and Yang, Yun},
  journal={IEEE Transactions on Information Theory},
  volume={67},
  number={6},
  pages={4223--4238},
  year={2021},
  publisher={IEEE}
}

@article{loffler2021optimality,
  title={Optimality of spectral clustering in the Gaussian mixture model},
  author={L{\"o}ffler, Matthias and Zhang, Anderson Y and Zhou, Harrison H},
  journal={Annals of Statistics},
  volume={49},
  number={5},
  pages={2506--2530},
  year={2021},
  publisher={Institute of Mathematical Statistics}
}

@misc{poker_hand_158,
  author       = {Cattral, Robert and Oppacher, Franz},
  title        = {{Poker Hand}},
  year         = {2002},
  howpublished = {UCI Machine Learning Repository}
}

@article{budanova2025penalized,
  title={Penalized estimation of finite mixture models},
  author={Budanova, Sofya},
  journal={Journal of Econometrics},
  volume={249},
  pages={105958},
  year={2025},
  publisher={Elsevier}
}

@article{tibshirani2001estimating,
  title={Estimating the number of clusters in a data set via the gap statistic},
  author={Tibshirani, Robert and Walther, Guenther and Hastie, Trevor},
  journal={Journal of the royal statistical society: series b (statistical methodology)},
  volume={63},
  number={2},
  pages={411--423},
  year={2001},
  publisher={Wiley Online Library}
}

@article{huang2017model,
  title={Model selection for Gaussian mixture models},
  author={Huang, Tao and Peng, Heng and Zhang, Kun},
  journal={Statistica Sinica},
  pages={147--169},
  year={2017},
  publisher={JSTOR}
}

@article{hamerly2003learning,
  title={Learning the k in k-means},
  author={Hamerly, Greg and Elkan, Charles},
  journal={Advances in Neural Information Processing Systems},
  volume={16},
  year={2003}
}

@article{malsiner2016model,
  title={Model-based clustering based on sparse finite Gaussian mixtures},
  author={Malsiner-Walli, Gertraud and Fr{\"u}hwirth-Schnatter, Sylvia and Gr{\"u}n, Bettina},
  journal={Statistics and Computing},
  volume={26},
  number={1},
  pages={303--324},
  year={2016},
  publisher={Springer}
}

@inproceedings{pelleg2000x,
  title={X-means: Extending k-means with efficient estimation of the number of clusters.},
  author={Pelleg, Dan and Moore, Andrew W and others},
  booktitle={ICML},
  volume={1},
  pages={727--734},
  year={2000},
  organization={Stanford, CA}
}

@article{akaike2003new,
  title={A new look at the statistical model identification},
  author={Akaike, Hirotugu},
  journal={IEEE Transactions on Automatic Control},
  volume={19},
  number={6},
  pages={716--723},
  year={2003},
  publisher={Ieee}
}

@article{schwarz1978estimating,
  title={Estimating the dimension of a model},
  author={Schwarz, Gideon},
  journal={Annals of Statistics},
  pages={461--464},
  year={1978},
  publisher={JSTOR}
}

@article{melnykov2012initializing,
  title={Initializing the EM algorithm in Gaussian mixture models with an unknown number of components},
  author={Melnykov, Volodymyr and Melnykov, Igor},
  journal={Computational Statistics \& Data Analysis},
  volume={56},
  number={6},
  pages={1381--1395},
  year={2012},
  publisher={Elsevier}
}

@article{feng2006pg,
  title={PG-means: learning the number of clusters in data},
  author={Feng, Yu and Hamerly, Greg},
  journal={Advances in Neural Information Processing Systems},
  volume={19},
  year={2006}
}

@article{teklehaymanot2018bayesian,
  title={Bayesian cluster enumeration criterion for unsupervised learning},
  author={Teklehaymanot, Freweyni K and Muma, Michael and Zoubir, Abdelhak M},
  journal={IEEE Transactions on Signal Processing},
  volume={66},
  number={20},
  pages={5392--5406},
  year={2018},
  publisher={IEEE}
}

@article{constantinopoulos2006bayesian,
  title={Bayesian feature and model selection for Gaussian mixture models},
  author={Constantinopoulos, Constantinos and Titsias, Michalis K and Likas, Aristidis},
  journal={IEEE Transactions on Pattern Analysis and Machine Intelligence},
  volume={28},
  number={6},
  pages={1013--1018},
  year={2006},
  publisher={IEEE}
}

@article{fisher1936use,
  title={The use of multiple measurements in taxonomic problems},
  author={Fisher, Ronald A},
  journal={Annals of Eugenics},
  volume={7},
  number={2},
  pages={179--188},
  year={1936},
  publisher={Wiley Online Library}
}

@article{campbell1974multivariate,
  title={A multivariate study of variation in two species of rock crab of the genus Leptograpsus},
  author={Campbell, NA and Mahon, RJ},
  journal={Australian Journal of Zoology},
  volume={22},
  number={3},
  pages={417--425},
  year={1974},
  publisher={CSIRO Publishing}
}

@article{hull2002database,
  title={A database for handwritten text recognition research},
  author={Hull, Jonathan J.},
  journal={IEEE Transactions on Pattern Analysis and Machine Intelligence},
  volume={16},
  number={5},
  pages={550--554},
  year={2002},
  publisher={IEEE}
}

@article{fop2018variable,
author = {Michael Fop and Thomas Brendan Murphy},
title = {{Variable selection methods for model-based clustering}},
volume = {12},
journal = {Statistics Surveys},
number = {none},
publisher = {Amer. Statist. Assoc., the Bernoulli Soc., the Inst. Math. Statist., and the Statist. Soc. Canada},
pages = {18 -- 65},
year = {2018}}

@article{gormley2023model,
  title={Model-based clustering},
  author={Gormley, Isobel Claire and Murphy, Thomas Brendan and Raftery, Adrian E},
  journal={Annual Review of Statistics and Its Application},
  volume={10},
  number={1},
  pages={573--595},
  year={2023},
  publisher={Annual Reviews}
}

@article{bishop2006pattern,
  title={Pattern recognition and machine learning},
  author={Bishop, Christopher M and Nasrabadi, Nasser M},
  volume={4},
  number={4},
  year={2006},
  publisher={Springer}
}

@inproceedings{corduneanu2001variational,
  title={Variational Bayesian model selection for mixture distributions},
  author={Corduneanu, Adrian and Bishop, Christopher M},
  booktitle={Proceedings Eighth International Conference on Artificial Intelligence and Statistics},
  pages={27--34},
  year={2001},
  organization={Morgan Kaufmann}
}

@article{mclachlan2000finite,
  title={Finite mixture models},
  author={McLachlan, Geoffrey J and Peel, David},
  year={2000},
  publisher={John Wiley \& Sons}
}

@inproceedings{lindsay1995mixture,
  title={Mixture models: theory, geometry, and applications},
  author={Lindsay, Bruce G},
  year={1995},
  organization={Ims}
}

@article{dempster1977maximum,
  title={{Maximum likelihood from incomplete data via the EM algorithm}},
  author={Dempster, Arthur P and Laird, Nan M and Rubin, Donald B},
  journal={Journal of the Royal Statistical Society: Series B (Methodological)},
  volume={39},
  number={1},
  pages={1--22},
  year={1977},
  publisher={Wiley Online Library}
}

\end{document}